%% file: main.tex
\documentclass{article}

\usepackage{microtype}
\usepackage{graphicx}
\usepackage{booktabs} 

\usepackage{hyperref}

\usepackage{url}            
\usepackage{amsfonts}       
\usepackage{amsthm}
\usepackage{mathrsfs}
\usepackage{nicefrac}       
\usepackage{microtype}      
\usepackage{booktabs} 

\usepackage[utf8]{inputenc} 
\usepackage[T1]{fontenc}    

\usepackage{tikz}
\usepackage{amsmath}
\usepackage{amssymb}
\usepackage{graphicx}
\usepackage{tabularx}
\usepackage{mathtools}
\usepackage{tcolorbox}
\usepackage{cancel}
\usepackage{pifont} 

\usepackage{multirow}
\usepackage{color}
\usepackage{mathtools}
\usepackage{wrapfig}
\usepackage{caption}
\usepackage{bbm}
\usepackage{bm}
\usepackage{esvect}
\usepackage[normalem]{ulem}

\usepackage{xcolor} %
\usepackage{xargs}
\usepackage{rotating}
\usepackage{longtable}
\usepackage{adjustbox}

\usepackage{cleveref}
\usepackage{subcaption}
\usepackage{enumitem}

\newcolumntype{Y}{>{\raggedright\arraybackslash}p}
	
\usepackage{contour}
\usepackage{ulem}

\contourlength{1pt}

\contourlength{0.8pt}

\newcommand{\myuline}[1]{%
  \uline{\phantom{#1}}%
  \llap{\contour{white}{#1}}%
}

\hypersetup{
	plainpages=false,
	colorlinks=true,              
	linkcolor=blue,               
	anchorcolor=blue,             
	citecolor=blue,               
	filecolor=blue,               
	pagecolor=blue,               
	urlcolor=blue,                
	pdfview=FitH,                 
	pdfstartview=FitH,            
	pdfpagelayout=SinglePage      
}

\input{doc/math_commands.tex}

\usepackage[final]{doc/neurips_2021}

\title{Multi-Objective SPIBB: Seldonian Offline Policy Improvement with Safety Constraints in Finite MDPs}

\author{%
  Harsh Satija\\
  McGill University, Mila \\
  \texttt{harsh.satija@mail.mcgill.ca} \\
   \And
  Philip S. Thomas \\ 
  University of Massachusetts \\ 
  \texttt{pthomas@cs.umass.edu} \\
   \AND
   Joelle Pineau \\
   McGill University, Mila, Facebook AI Research \\
   \texttt{jpineau@cs.mcgill.ca} \\
   \And
   Romain Laroche \\ 
   Microsoft Research \\
   \texttt{romain.laroche@microsoft.com}
}

\begin{document}

\maketitle


\begin{abstract}
We study the problem of Safe Policy Improvement (SPI) under constraints in the offline Reinforcement Learning (RL) setting. We consider the scenario where: (i) we have a dataset collected under a known baseline policy, (ii) multiple reward signals are received from the environment inducing as many objectives to optimize. 
We present an SPI formulation for this RL setting that takes into account the preferences of the algorithm's user for handling the trade-offs for different reward signals while ensuring that the new policy performs at least as well as the baseline policy along each individual objective. 
We build on traditional SPI algorithms and propose a novel method based on Safe Policy Iteration with Baseline Bootstrapping~\citep[SPIBB,][]{laroche2017safe} 
that provides high probability guarantees on the performance of the agent in the true environment.
We show the effectiveness of our method on a synthetic grid-world safety task as well as in a real-world critical care context to learn a policy for the administration of IV fluids and vasopressors to treat sepsis.

\end{abstract}

\input{doc/sections/introduction}

\input{doc/sections/related_work}

\input{doc/sections/background}


\input{doc/sections/tabular_experiments}

\input{doc/sections/real_world_experiments}

\input{doc/sections/discussion}

\input{doc/sections/acknowledgements}


\bibliographystyle{apalike}
\bibliography{doc/lib}

\section*{Checklist}

\begin{enumerate}

\item For all authors...
\begin{enumerate}
  \item Do the main claims made in the abstract and introduction accurately reflect the paper's contributions and scope?
    \answerYes{See \Cref{sec:spibb-w-constraints,sec:hcpi-w-constraints} for the methodology and theoretical claims, and \Cref{sec:synthetic-experiments,sec:sepsis-experiments} for the empirical results.}
  
  \item Did you describe the limitations of your work?
    \answerYes{See \Cref{sec:introduction} for the limitations and scope of this work.}
  
  \item Did you discuss any potential negative societal impacts of your work?
    \answerYes{
    As we mentioned in \Cref{sec:introduction,sec:problem-formulation}, our goal is to maximize the objective specified by the user while ensuring that the solution policy avoids causing harmful effects after deployment in the true environment in comparison to the existing baseline policy. 
    We aim to bridge the gap between traditional RL methods and high-stake real-world applications, but, as with any general technology, we acknowledge that some RL applications can have the potential of misuse, and our methods do not prevent that.
    }

  \item Have you read the ethics review guidelines and ensured that your paper conforms to them?
    \answerYes{We discuss the societal impacts in the point above. We use pre-existing publicly available data and libraries and give more details about them in the Point 4 below.}
\end{enumerate}

\item If you are including theoretical results...
\begin{enumerate}
  \item Did you state the full set of assumptions of all theoretical results?
    \answerYes{See \Cref{sec:setting} for the assumption regarding access to baseline policy.}
	\item Did you include complete proofs of all theoretical results?
    \answerYes{The complete proofs are provided in the  \Cref{app:spibb-additional-details} and \Cref{app:hcpi-details}.}
\end{enumerate}

\item If you ran experiments...
\begin{enumerate}
  \item Did you include the code, data, and instructions needed to reproduce the main experimental results (either in the supplemental material or as a URL)?
    \answerYes{The code required to produce the results is provided in the supplementary material.}
    
  \item Did you specify all the training details (e.g., data splits, hyperparameters, how they were chosen)?
    \answerYes{The training details are mentioned in  \Cref{sec:synthetic-experiments,sec:sepsis-experiments} in the main text, and  \Cref{app:additional-details-for-synthetic-exp,app:sepsis-details}.}
    
	\item Did you report error bars (e.g., with respect to the random seed after running experiments multiple times)?
    \answerYes{The details about the error bars are provided in \Cref{sec:synthetic-experiments,sec:sepsis-experiments}.}
    
	\item Did you include the total amount of compute and the type of resources used (e.g., type of GPUs, internal cluster, or cloud provider)?
    \answerYes{Details about the compute and resources can be found in \cref{app:additional-details-for-synthetic-exp,app:sepsis-details}.
    }
\end{enumerate}

\item If you are using existing assets (e.g., code, data, models) or curating/releasing new assets...
\begin{enumerate}
  \item If your work uses existing assets, did you cite the creators?
    \answerYes{See \Cref{sec:synthetic-experiments,sec:sepsis-experiments} for the appropriate references.}
  
  \item Did you mention the license of the assets?
    \answerYes{We use the \citep[MIMIC-III,][]{johnson2016mimic} dataset and provide the appropriate reference. The explicit link to the license is here: \url{https://physionet.org/content/mimiciii/view-license/1.4/}. }
  
  \item Did you include any new assets either in the supplemental material or as a URL?
    \answerYes{We provide the accompanying code for running the experiments in the supplemental material.}
    
  \item Did you discuss whether and how consent was obtained from people whose data you're using/curating?
    \answerNA{We are using publicly available libraries and dataset.}
    
  \item Did you discuss whether the data you are using/curating contains personally identifiable information or offensive content?
    \answerNA{We are using the publicly available dataset  \citep[MIMIC-III,][]{johnson2016mimic} that already deidentifies the data in accordance with Health Insurance Portability and Accountability Act (HIPAA) standards using structured data cleansing and date shifting. We refer to the original text for more details on the deidentification process.
    }
\end{enumerate}

\item If you used crowdsourcing or conducted research with human subjects...
\begin{enumerate}
  \item Did you include the full text of instructions given to participants and screenshots, if applicable?
    \answerNA{}
  \item Did you describe any potential participant risks, with links to Institutional Review Board (IRB) approvals, if applicable?
    \answerNA{}
  \item Did you include the estimated hourly wage paid to participants and the total amount spent on participant compensation?
    \answerNA{}
\end{enumerate}

\end{enumerate}



\clearpage
\onecolumn
\appendix

\input{doc/appendix/naive_obj_construction}



\input{doc/appendix/spibb_theorems}
\input{doc/appendix/hcpi_details}
\input{doc/appendix/cmdp_extra_details}

\input{doc/appendix/sepsis_additional_details}


\end{document}

%% file: doc/math_commands.tex
\usepackage{amsmath,amsfonts,bm}
\usepackage{xifthen}

\theoremstyle{definition}

\newtheorem*{remark}{Remark} 

\newtheorem{prop}{Proposition}[section]

\newtheorem{assumption}{Assumption}[section]

\DeclareMathOperator*{\argmax}{arg\,max}

\DeclarePairedDelimiter\set\{\}
\newcommand*{\mytop}{\mathrel{\scalebox{0.5}{$\top$}}}

\newcommand*{\bmg}{\bm{\gamma}}
\newcommand*{\bml}{\bm{\lambda}}

\newcommand{\X}{\mathcal{X}}
\newcommand{\A}{\mathcal{A}}

\newcommand{\Dist}{\mathscr{P}}
\newcommand{\D}{\mathcal{D}}
\newcommand{\Real}{\mathbb{R}}

\newcommand{\E}{\mathop{\mathbb{E}}}

\newcommand{\Pstar}{p^{\star}}
\newcommand{\Rstar}{\bm{r}^{\star}}

\newcommand{\rmax}{r_{\mytop}}

\newcommand{\mstar}{m^{\star}}
\newcommand{\mhat}{\hat{m}}
\newcommand{\mopt}{m^{\star}}

\newcommand{\Phat}{\hat{p}}
\newcommand{\Rhat}{\hat{\bm{r}}}

\newcommand{\V}[3]{ %
    \ifthenelse{\isempty{#3}}%
    {V^{#1}(#2)}
    {V^{#1}_{#3}(#2)}%
}

\newcommand{\Q}[3]{
    \ifthenelse{\isempty{#3}}
    {Q^{#1}(#2)}
    {Q^{#1}_{#3}(#2)}%
}

\newcommand{\Adv}[3]{
    \ifthenelse{\isempty{#3}}
    {A^{#1}(#2)}
    {A^{#1}_{#3}(#2)}%
}

\newcommand{\J}[3]{
    \ifthenelse{\isempty{#3}}
    {\mathcal{J}^{#1}_{#2}}
    {\mathcal{J}^{#1}_{#3,#2}}
}

\newcommand{\val}[4]{ %
    \ifthenelse{\isempty{#3}}%
    {v^{#1}_{#2}(#4)}
    {v^{#1}_{#3,#2}(#4)}
}

\newcommand{\qval}[4]{
    \ifthenelse{\isempty{#3}}
    {q^{#1}_{#2}(#4)}
    {q^{#1}_{#3,#2}(#4)}
}
\DeclareMathOperator*{\advantage}{Adv}

\newcommand{\adv}[4]{
    \ifthenelse{\isempty{#3}}
    {\advantage^{#1}_{#2}(#4)}
    {\advantage^{#1}_{#3,#2}(#4)}
}

\newcommand{\pib}{\pi_{b}}
\newcommand{\piopt}{\pi^{*}}

\newcommand{\pr}{\text{Pr}}
\newcommand{\IS}{\text{IS}}
\newcommand{\CI}{\text{CI}}


%% file: doc/sections/introduction.tex
\section{Introduction}
\label{sec:introduction}
Reinforcement Learning (RL) as a paradigm for sequential decision-making \citep{sutton1988learning} has shown tremendous success in a variety of simulated domains \citep{mnih2015human, silver2017mastering, OpenAI_dota}.
However, there are still quite a few challenges between the traditional RL research and real-world tasks.
Most of these challenges stem from assumptions that are rarely satisfied in practice \citep{dulac2019challenges}, or the inability of the algorithm's user to specify the desired behavior of the agent without being a domain expert \citep{Thomas2019}.
We focus on the real-world application point of view and posit the following requirements:
\begin{itemize}[topsep=0pt, leftmargin=*]
    \item \textbf{Multiple reward functions:} Traditional RL methods assume a single scalar reward is present in the environment.
    However, most real-world tasks, have multiple (possibly conflicting) objectives or constraints that need to be taken into consideration together, such as the signals related to the safety (physical well-being of the agent or the environment), budget utilization (energy or maintenance costs), etc.
    
    \item \textbf{Stakeholder control of the trade-off:} 
    The ML practitioners should have the ability to control the different trade-offs the agent is making and choose the one they consider best for the task at hand. 
    
    \item \textbf{Offline setting:} 
    In many real-world domains (e.g., healthcare, finance or autonomous vehicles), there is an abundance of data, collected under a sub-optimal policy, but training the agent directly via interactions with the environment is expensive and risky.
    We assume that we only have access to a dataset of past trajectories that can be used for training \citep{lange2012batch}. 
    
    \item \textbf{Preventing unintended behavior:} We want the agent to be robust to 
    both extrapolation errors from offline RL and misaligned objectives that are poor proxy of the user's intentions and algorithm's actual performance \citep{ng1999policy, amodei2016concrete}.
    We consider the case where the user can specify undesirable behavior in the context of the performance observed in the batch. 
    
    \item \textbf{Practical guarantees:}
    We want guarantees about the undesirable behavior that might be caused by the agent in the real-world. We care about the results that can be obtained using the finite amount of samples we have in the batch, and aim to provide some measure of confidence in deploying the agents in the environment. 
\end{itemize}

To achieve this set of properties, we adopt the Seldonian framework~\citep{Thomas2019}, which is a general algorithm design framework that allows high-confidence guarantees for constraint satisfaction in a multi-objective setting. 
Based on the above specifications, we seek to answer the question:
\textit{if we are given a batch of data collected under some (suboptimal) behavioral policy and some user preference, can we build a policy improvement algorithm that returns a policy with practical high-confidence guarantees on the performance of the policy w.r.t. the behavioral policy?}

We acknowledge that there are other important challenges in RL, 
such as partial observability, safe exploration, non-stationary environments and function approximation in high-dimensional spaces, that also stand in the way of making RL a more applicable paradigm. These challenges are beyond the scope of this work, which should rather be thought of as taking a step towards this broader goal.

In \Cref{sec:related-work}, we present our contribution positioned with respect to other related work. 
In \Cref{sec:methodology}, we formalize the setting and then extend traditional SPI algorithms to this setting. 
We then show it is possible to extend the previous work on Safe Policy Iteration (SPI), particularly Safe Policy Iteration with Baseline Bootstrapping~\citep[SPIBB,][]{laroche2017safe},  for the design of agents that satisfy the above requirements. 
We show that the resulting algorithm is theoretically-grounded and provides practical high-probability guarantees. 
We extensively test our approach on a synthetic safety-gridworld task in \Cref{sec:synthetic-experiments} and show that the proposed algorithm achieves better data efficiency than the existing approaches.
Finally, we show its benefits on a critical-care task in  \Cref{sec:sepsis-experiments}. 
The accompanying codebase is available at \url{https://github.com/hercky/mo-spibb-codebase}.

%% file: doc/sections/related_work.tex
\section{Related work}
\label{sec:related-work}

\textbf{Multi-Objective RL (MORL):}
Traditional multi-objective approaches \citep{mannor2004geometric, roijers2013survey, liu2014multiobjective} focus on finding the Pareto-frontier of optimal reward functions that gives all the possible trade-offs between different objectives. The user can then select a policy from the solution set according to their arbitrary preferences. In practice, an alternate trial and error based approach of scalarization is used to transform the multiple reward functions into a scalar reward based on preferences across objectives (usually, by taking a linear combination). 
Most traditional MORL approaches have focused on the online, interactive settings where the agent has access to the environment. While some recent approaches are based on off-policy learning methods \citep{lizotte2012linear, van2014multi,yang2019generalized,abdolmaleki2020distributional}, they lack guarantees. In contrast, our work focuses exclusively on learning in the offline setting and gives high-probability guarantees on the performance in the environment.

\textbf{Constrained-RL:} 
RL under constraints frameworks, such as Constrained MDPs~\citep[CMDPs,][]{altman1999constrained}, present an alternative way to define preferences in the form of constraints over policy's returns. 
Here, the user assigns a single reward function to be the primary objective (to maximize) and hard constraints are specified for the others.
The major limitation of this setting is that it assumes the thresholds for the constraints are known a priori. 
%
\citet{le2019batch} study offline policy learning under constraints
and provide performance guarantees w.r.t. the optimal policy, but their work relies on the concentrability assumption \citep{munos2003error}.

Concentrability is a strong assumption that upper bounds the ratio between the future state-action distributions of any non-stationary policy and the baseline policy under which the dataset was generated by some constant. From a practical perspective, it is unclear how to get a tractable estimate of this constant, as the space of future state-action distributions of non-stationary policies is vast. Thus, this constant can be arbitrarily huge, potentially even infinite when the baseline policy fails to cover the support of the space of all non-stationary policies (such as in the low-data regime), leading to the performance bounds given by these methods to blow up (and even be unbounded). 
Additionally, the guarantees in \cite{le2019batch} are only valid with respect to the performance of the optimal policy.
In this work, we instead focus on the performance guarantees based on returns observed in the dataset, as it does not require making any of the above assumptions.

\textbf{Reward design:} 
Reward-design \citep{sorg2010internal} and reward-modelling approaches \citep{christiano2017deep, littman2017environment, leike2018scalable} focus on designing suitable reward functions that are consistent with the user's intentions. These approaches rely heavily on the human or simulator feedback, and thus do not carry over easily to the offline setting.

\textbf{Seldonian-RL (and Safe Policy Improvement):} 
The Seldonian framework \citep{Thomas2019} is a general algorithm design framework that allows the user to design ML algorithms that can avoid undesirable behavior with high-probability guarantees. 
In the context of RL, the Seldonian framework allows to design policy optimization problems with multiple constraints,
where the solution policies satisfy the constraints with high-probability. 
In the offline-RL setting, SPI refers to the objective of guaranteeing a performance improvement over the baseline with high-probability guarantees~\citep{thomas2015highImprovement, petrik2016safe, laroche2017safe}. Therefore, SPI algorithms are a specific setting that falls in the general category of Seldonian-RL algorithms. 

We focus on two categories of SPI algorithms that provide practical error bounds on safety: SPIBB \citep{laroche2017safe} that provides Bayesian bounds and HCPI \citep{thomas2015highImprovement, thomas2015highEvaluation} that provides frequentist bounds.
SPIBB methods constrain the change in the policy according to the local model uncertainty. 
SPIBB has been formulated in the context of a single reward function, and as such does not handle multiple rewards and by extension also lacks the ability for the user to specify preferences. 
Our primary focus is to provide a construction for extending the SPIBB methodology to the multi-objective setting that handles user preferences and provides high-probability guarantees.

Instead of relying on model uncertainty, HCPI methods utilize the high-confidence lower bounds on the Importance Sampling (IS) estimates of a target policy's performance to ensure safety guarantees. 
HCPI has been applied to solve Seldonian optimization problems for constrained-RL setting using an enumerable policy class. \citet{Thomas2019} suggested using HCPI for the MORL setting, and we build on that idea. Particularly, we show how HCPI can be implemented with stochastic policies in the context of our setting with user preferences and baseline constraints.

%% file: doc/sections/background.tex
\section{Methodology}
\label{sec:methodology}

\subsection{Setting}
\label{sec:setting}

We consider the setting where the agent's interactions with the environment can be modelled as a Markov Decision Process \citep[MDP,][]{bellman1957markovian}. 
Let $\X$ and $\A$ respectively be the (finite) state and action spaces. 
Let $\Pstar: \X \times \A \rightarrow \Dist(\X)$ denote the true (unknown) transition probability function, where $\Dist(\X)$ denotes the set of probability distributions on $\X$. 
Without loss of generality, we assume that the process deterministically begins in the state $x_0$.
We define $[N]$ to be the set $\{0,1,\dots,N-1\}$ for any positive integer $N$.
Let there be $d$ different reward signals and
$\Rstar=\left\{r_k\right\}_{k\in[d]}: \X \times \A \rightarrow [- \rmax, \rmax]^d$ be the true (unknown) stochastic multi-reward signal.\footnote{Costs, which are meant to be minimized, can be expressed as negative rewards.}
Finally, $\bmg=\left\{\gamma_k\right\}_{k\in[d]} \in [0,1)^d$ is the multi-discount-factor.

The MDP, $\mstar$, can now be defined with the tuple $(\X, \A, \Pstar, \Rstar, \bmg, x_0)$. A policy $\pi: \X \rightarrow \Dist(\A)$ maps a state to a distribution over actions. We denote by $\Pi$ the set of stochastic policies. We consider the infinite horizon discounted return setting.
For any $k\in[d]$, the $k$\textsuperscript{th} reward value function $\val{\pi}{k}{m}{x} : \X \rightarrow \Real$ denotes the expected discounted sum of rewards when when following policy $\pi$ in an MDP $m$ starting from state $x$.
Analogously, we define the state-action value functions for performing action $a$ in state $x$ in MDP $m$ under $\pi$ for rewards as $\qval{\pi}{k}{m}{x,a}$. Let $\adv{\pi}{k}{m}{x,a} = \qval{\pi}{k}{m}{x,a} - \val{\pi}{k}{m}{x}$ denote the corresponding advantage function.
The expected return of policy $\pi$ w.r.t. the $k$\textsuperscript{th} reward in the true MDP $\mstar$ is denoted by $\J{\pi}{k}{\mstar} = \val{\pi}{k}{\mstar}{x_0} = \E_{\pi, \mstar} [ \sum_{t=0}^{\infty} \gamma_k^t R_{k,t} \mid X_0=x_0]$, where action $A_t\sim \pi(\cdot\mid X_t)$, immediate reward $R_{k,t}\sim r^\star_k(\cdot\mid X_t,A_t)$, and state $X_{t+1}\sim \Pstar(\cdot\mid X_t,A_t)$.

We consider the offline setting, where instead of having access to the environment we have a pre-collected dataset 
of trajectories denoted by $\D = \set{\tau_i}_{i\in[|\D|]}$, where $|\D|$ denotes the number of trajectories in the dataset. A trajectory $\tau$ of length $T$ is an ordered set of transition tuples of the form $\tau = \set{x_i, a_i, x'_i, \bm{r}_i}_{i\in[T]}$, where $x'_i$ denotes the state at the next time-step.
We denote the Maximum Likelihood Estimation (MLE) of the MDP with $\mhat = (\X, \A, \Phat, \Rhat, \bmg, x_0)$, where $\Phat$ and $\Rhat$ denote the transition and reward models estimated from the dataset's statistics.

\begin{assumption}[Baseline policy]
We assume that we have access to the policy that generated the dataset. We call such policy the baseline policy and denote it by $\pib$.~\footnote{\citet{simao2020} proved that SPIBB/Soft-SPIBB bounds may be obtained with an estimate of $\pib$.
}
\end{assumption}

\subsection{Problem formulation}
\label{sec:problem-formulation}

We consider safe policy improvement with respect to the baseline according to the $d$ dimensions of the multi-objective setting.
Therefore, under a Bayesian approach, we search for target policies such that they perform better (up to a precision error $\zeta$) than the baseline along every objective function with high probability $1 - \delta$, where $\zeta$ and $\delta$ are hyper-parameters controlled by the user, denoting the risk that the practitioner is willing to take. We denote by $\Pi_\textsc{a}$ the set of admissible policies that satisfy:
\begin{align}
    \label{eq:general-safety-constraints}
    \mathbb{P}\left(\forall k\in[d], \J{\pi}{k}{\mstar}-\J{\pi_b}{k}{\mstar}>-\zeta \Big| \D \right)>1-\delta .
\end{align}

In the multi-objective case, there does not exist a single optimal value, but a Pareto frontier of optimal values. One way to evaluate the MORL problems is via the \textit{multiple-policy} approaches \citep{vamplew2011empirical, roijers2013survey} that compute the policies that approximate the true optimal Pareto-frontier. However, note that optimality and safety are contradicting objectives. It is not clear how (and if) one can make claims about optimality in the offline setting without bringing in additional unrealistic assumptions (\Cref{sec:related-work}, MORL). Instead, we take an alternate approach inspired by another category of MORL methods called \textit{single-policy} \citep{roijers2013survey, van2014multi} where the trade-offs between different objectives are explicitly controlled by the user via providing a scalarization or preferences over objectives. 
We assume the user preference $\bml=\left\{\lambda_k\right\}_{k\in[d]}$ is given as an input to our algorithms, and is used for scalarization of the objectives, where $\lambda_k \in \Real^{+}$. 
Our objective therefore becomes
\begin{align}
\label{eq:general-task-objective}
    \argmax_{\pi \in \Pi_\textsc{a}} &\quad \sum_{k\in[d]} \lambda_k \J{\pi}{k}{\mstar} . 
\end{align}

The above formulation gives freedom to the user in terms of what particular quantity they want to optimize via $\bml$, but still ensures that the solution policy performs as well as the baseline policy across all $d$ objectives.
Note that our explicit goal is to maximize the objective specified by the user. However, the user might make mistakes in specifying this objective (\Cref{sec:related-work}, Reward design), and the above formulation offers guarantees that prevent deteriorating the performance of the policy across any of the $d$ objectives. 
This allows the user to to experiment with different reward design strategies in safety-critical settings without worrying about the risks of ill-defined scalarizations.
A na\"ive approach would be applying the user scalarization to also define the safety constraints. However, this construction fails to prevent undesirable behavior for the individual objectives (shown in \Cref{app:naive-construction}).


\subsection{Multi-Objective SPIBB (MO-SPIBB)}
\label{sec:spibb-w-constraints}

Robust MDPs~\citep{Iyengar2005,Nilim2005} can be regarded as an approximation of the Bayesian formulation by partitioning the MDP space $\mathcal{M}$ into two subsets: the subset of plausible MDPs $\Xi$  and the subset of implausible MDPs. The plausible set is classically constructed from concentration bounds over the reward and transition function:
\begin{align*}
    \Xi = &\left\{m,
	\textnormal{ s.t. } \forall x,a, 
	\begin{array}{ll}
	\lVert p(\cdot|x,a)-\hat{p}(\cdot|x,a)\rVert_1 \leq e(x,a),\\
	\lVert \bm{r}(x,a)-\hat{\bm{r}}(x,a)\rVert_\infty \leq e(x,a)r_{\mytop} \end{array}\right\},
\end{align*}
where $e$ is an upper bound on the state-action error function of the model that are classically obtained with concentration bounds, such that the true environment $m^{\star}\in\Xi$ with high probability $1-\delta$.
In the single objective framework, \citet{laroche2017safe} empirically show that optimising the worst-case performance policy in $\Xi$ provides policies that are too conservative. \citet{petrik2016safe} prove that it is NP-hard to find the policy $\pi$ that maximises the worst-case policy improvement over $\Xi$. 

Instead, the SPIBB methodology \citep{laroche2017safe} consists in searching for a policy that maximizes the safe policy improvement in the MLE MDP, under some policy constraints:
SPIBB and Soft-SPIBB \citep{nadjahi2019safe} policy search constraints both revolve around the idea that we must only consider policies for which the policy improvement may be accurately estimated.   
Using $\pi_b$ as reference, SPIBB allows policy changes only in state-action pairs for which more than $n_\wedge$ samples have been collected. Soft-SPIBB extends this by applying soft constraints that allow slight changes in the policy for the uncertain state-action pairs, which are controlled by an error bound related to model uncertainty. As such, on low-confidence transitions, this class of methods provides a mechanism that prevents the agent from deviating too much from $\pib$. 
In this work, we build on Soft-SPIBB because it has yielded better empirical results.
Formally, its constraint on the policy class is defined by:
\begin{equation*}
\label{eq:spibb-policy-constraint}
    \Pi_\textsc{s} = \left\{\pi, \text{ s.t. } \forall x, \sum_a e(x,a)\ |\pi(a|x) - \pib(a|x)| \leq \epsilon \right\},
\end{equation*}
where $\epsilon$ is a hyper-parameter that controls the deviation from the baseline policy.

We define $\qval{\pi}{\bml}{m}{x,a}=\sum_{k\in[d]} \lambda_k \qval{\pi}{k}{m}{x,a}$ to be the state-action value function associated with the linearized $\bml$ parameters. The same notation extension is used for $\val{\pi}{\bml}{m}{x,a}$ and $\J{\pi}{\bml}{m}$. The application of Soft-SPIBB to multi-objective safe policy improvement is therefore direct:
\begin{align}
\label{eq:mo-spibb-objective}
    \argmax_{\pi \in \Pi_\textsc{a}\cap\Pi_\textsc{s}} \J{\pi}{\bml}{\mhat} , 
\end{align}
which is always realizable since $\pib\in\Pi_\textsc{a}\cap\Pi_\textsc{s}$.

We show that the construction of the plausible set required for the application of SPIBB is technically sound by deriving the concentration bounds for the multi-objective case. 
In Appendix \ref{app:error_bounds}, we show with Hoeffding's inequality that $e$ grows as the square root of the logarithm of $d$ (the number of reward functions), \textit{i.e.} almost imperceptibly. From there, all the SPIBB theoretical results from \citet{laroche2017safe,nadjahi2019safe,simao2020} may be generalized at a negligible SPI guarantee cost to the multi-objective setting, by applying their theorems separately to every objective function.

Now, the problem in \Cref{eq:mo-spibb-objective} can be transformed into a policy improvement procedure that solves for every state $x \in \X$ the following optimization problem\footnote{In practice, we also need to check $\forall x$ that $\pi(\cdot\mid x)$ is a valid probability distribution: positive and sums to 1.}:
\begin{align*}
    \label{eq:s-opt}
    \pi_\textsc{s-opt} &= \argmax_{\pi \in \Pi} \langle \pi(\cdot|x) , \qval{\pi}{\bml}{\mhat}{x,\cdot} \rangle \quad  \tag{$\texttt{S-OPT}$} \\  
    \text{s.t.} \quad    
    &\sum_{a \in \A} e(x,a)\ |\pi(a|x) - \pib(a|x)| \leq \epsilon, \tag{$\pi\in\Pi_\textsc{s}$} \\ 
    &\forall k\in [d], \sum_{a \in \A} \pi(a|x) \adv{\pib}{k}{\mhat}{x, a} \geq 0. \tag{$\pi\in\Pi_\textsc{a}$}
\end{align*}


The above procedure requires us to make additional algorithmic modifications that are not present in the original SPIBB algorithms. 
In particular, we need to explicitly incorporate advantage constraints for safety-guarantees for the individual objectives (proof given in \Cref{app:spibb-need-of-advatangeous}). The classic single-objective SPIBB algorithms do
not need to check the advantage conditions because it is automatically guaranteed by the $\argmax$ and the fact that $\pib\in\Pi_\textsc{s}$.

Using the construction above, we directly get the following result on the performance guarantees for each objective function that satisfies the desired property in \Cref{eq:general-safety-constraints}: 

\begin{prop}
The policy $\pi$ returned from solving the \ref{eq:s-opt} satisfies the following property in every state $x \in \X$ with probability at least $(1 - \delta)$:
\begin{align}
    \forall k \in [d], \, \val{\pi}{k}{\mopt}{x} - \val{\pib}{k}{\mopt}{x} \geq -\frac{\epsilon v_{\text{max}}}{1-\gamma},
\end{align}
where $v_{\text{max}} \le \frac{\rmax}{1-\gamma}$ is the maximum of the value function.
\end{prop}
The proof is presented in \Cref{app:mo-spibb-prop}.
The solution of \ref{eq:s-opt} is computed by solving the Linear Program using standard solvers, such as 
\texttt{cvxpy} \citep{diamond2016cvxpy}. 
There is an increase in the computational cost proportional to the number of reward functions. Compared to Soft-SPIBB, the value and advantage functions estimation cost increases by a factor of $d$: respectively $\mathcal{O}(d|\X|^3)$ and $\mathcal{O}(d|\A||\X|^2)$.
There is $\mathcal{O}(|\mathcal{D}|)$ cost for estimating the error bounds, and we also require solving a Linear Program for each state that approximately amounts to an additional $\mathcal{O}(|\mathcal{X}||\mathcal{A}|^2(|\mathcal{A}|+d))$ steps to the total computational cost \citep{boyd2004convex}.

\begin{remark}[Extension to Constrained-RL]
The above methodology can also be extended to the Constrained-RL setting for offline policy improvement in general CMDPs. Recall that SPIBB algorithms offer guarantees in the form of: $v_t-v_b \geq \hat{v}_t - \hat{v}_b - \xi$, where $v_t$ and $v_b$ are respectively the true values of the target and baseline policies, $\hat{v}_t$ and $\hat{v}_b$ are their estimates in the MLE MDP, and $\xi$ is an error term due to parametric uncertainty. As a consequence, any constraint $c$ such that $c\leq v_b + \hat{v}_t - \hat{v}_b - \xi$ may be guaranteed ($v_b-\hat{v}_b$ may easily be bounded with Hoeffding's inequality), and when $c$ is larger, we can return no solution found as with other Seldonian algorithms.
\end{remark}


\input{doc/sections/hcpi_extension}

%% file: doc/sections/hcpi_extension.tex
\subsection{Multi-Objective HCPI (MO-HCPI)}
\label{sec:hcpi-w-constraints}

We briefly recall how the HCPI methodology \citep{thomas2015highImprovement, thomas2015highEvaluation} can be applied directly for solving the objective in \Cref{eq:general-task-objective}.
For a target policy, $\pi_t$, we use $\IS_{k}(\D, \pi_t, \pib)$ to denote the estimated returns for the $k$th reward component ($r_{k}$) using any IS based off-policy estimator \citep{precup2000eligibility}.  A high-confidence lower bound on $\J{\pi_t}{k}{\mopt}$ can be defined as:  
\begin{equation}
    \label{eq:hcope-R-lower-bound}
    \pr \Big( \J{\pi_t}{k}{\mopt} \ge \IS_{k}(\D, \pi_t, \pib) - \CI_{k}(\D, \delta/d) \Big) \ge 1 - \delta/d, 
\end{equation}
where $\CI_{k}(\D, \delta) \geq 0$ denotes the terms associated with the choice of concentration inequality employed (and typically $\lim_{|\D| \rightarrow \infty}\CI_{k}(\D, \delta) = 0$). 

The dataset $\D$ is first split into train ($\D_{tr}$) and test ($\D_{s}$) sets by the user. 
Let $\IS_{\bml}$ denote the IS estimator associated with the user-specified reward scalarization $\bml$. Given the user specified parameters: $\bml, \delta, \CI, \IS, \D_{tr}, \D_{s}$ and $\pib$, the policy improvement problem in \Cref{eq:general-task-objective} is transformed to the following optimization problem:
\begin{align*}
    \label{eq:h-opt}
    \pi_\textsc{h-opt} &= \argmax_{\pi \in \Pi} \IS_{\bml}(\D_{tr}, \pi, \pib) \tag{$\texttt{H-OPT}$}\\
    \text{s.t.} \quad
    &\forall k\in [d], \; \IS_{k}(\D_{s}, \pi, \pib) - \CI_k(\D_s, \delta/d) \ge
    \mu_k, 
\end{align*}
where $\mu_k$ denote the empirical returns for $r_k$ under $\pib$. 
%
The policy $\pi$ returned by \ref{eq:h-opt} will only violate the safety guarantees with probability at most $\delta$.
Proof of this claim and additional details are provided in \Cref{app:hcpi-details}.

Although we only focus on finite MDPs in this work, the HCPI based approach relies on IS estimates and therefore it can also be used for infinite MDPs or POMDPs. Unfortunately, the IS estimates are typically known to suffer from high variance \citep{guo2017using}. 
Furthermore, the optimization problem in \ref{eq:h-opt} is more challenging, and we need to resort to regularization based heuristics.

%% file: doc/sections/tabular_experiments.tex
\section{Synthetic Experiments}
\label{sec:synthetic-experiments}

The main benefits of working in a synthetic domain are: (i) we can evaluate the performance on the true MDP instead of relying on off-policy evaluation (OPE) methods, (ii) we have control over the quality of the dataset. We test both 
MO-SPIBB (\ref{eq:s-opt}) and MO-HCPI (\ref{eq:h-opt}) on a variety of parameters: the amount of data, quality of baseline and different user reward scalarizations. 

\textbf{Env details:} 
We take a standard CMDP benchmark \citep{leike2017ai, chow2018lyapunov} which consists of a $10\times10$ grid. From any state, the agent can move to the adjoining cells in the 4 directions using the 4 actions. 
The transitions are stochastic, with some probability $\alpha$ (generated randomly for each state-action for every environment instance) the agent is successfully able to reach the next state, and with $(1-\alpha)$ the agent stays in the current state. 
The agent starts at the bottom-right corner, and the goal is to reach the opposite corner (top-left). The pits are spawned randomly with some uniform probability ($\eta_{pit}=0.3$) for each cell.  
The reward vector consists of two rewards signals. A primary reward $r_0$ that is related the goal and is +1000.0 on reaching the goal and -1.0 at every other time-step. The secondary reward $r_1$ is related to pits, for which the agent gets -1.0 for any action taken in the pit.
The constraint threshold for this CMDP is $-2.0$ and $\gamma = 0.99$. Maximum length of an episode is $200$ steps. Therefore, the task objective is to reach the goal in the least number of steps, such that the agent does not spend more than $2$ time-steps in the pit cells. 

\textbf{Dataset collection procedure:} 
For every random CMDP generated, we first find the optimal policy $\piopt$ by using the procedure described in \Cref{app:cmdp-solver}. The baseline policy is generated using a convex combination of the optimal policy and a uniform random policy ($\pi_{rand}$), i.e., $\pib = \rho \piopt + (1 - \rho) \pi_{rand}$, where  $\rho$ controls how close $\pib$'s performance is to $\piopt$. Different datasets with varying sizes and $\rho$ are then collected under $\pib$ and given as input to the methods.

\textbf{Baselines:} 
We compare against the following baselines:
\begin{itemize}[leftmargin=*, topsep=0pt]
    \item \myuline{Linearized}: This baseline transforms the rewards into a single scalar using $\bml$ and then applies the traditional policy improvement methods on the linearized objective, i.e,  $\argmax_{\pi \in \Pi} \J{\pi}{\bml}{\mhat}$.
    
    \item \myuline{Adv-Linearized}: This method has the same objective as the Linearized baseline, with the additional constraints based on advantage estimators built from $\mhat$, i.e. $\forall x \in \X$:
    \begin{align}
        \argmax_{\pi \in \Pi} &\langle \pi(\cdot|x) , \qval{\pi}{\bml}{\mhat}{x,\cdot} \rangle \\   
        \text{s.t.} \quad    
        &\forall k\in [d], \; \sum_{a \in \A} \pi(a|x) \adv{\pib}{k}{\mhat}{x, a} \geq 0 . \nonumber
    \end{align}

\end{itemize}

\textbf{Evaluation:} 
Using $\mopt$, we can directly calculate the returns for any solution policy. Only tracking the scalarized objective can be misleading, so we track the following metrics:
\begin{itemize}[leftmargin=*, topsep=0pt,]
    \item \myuline{Improvement over $\pib$}: This denotes the difference between the scalarized return of the solution policy and the baseline policy, i.e., $\J{\pi}{\bml}{\mopt} - \J{\pib}{\bml}{\mopt}$.
    Mean improvement over $\pib$ captures on average improvement over $\pib$ in terms of the scalarized objective.
    
    \item \myuline{Failure-rate:} 
    The failure rate over $n$ runs captures the number of times, on average, the solution policy ends up violating the safety constraints in \Cref{eq:general-safety-constraints}, and thus performs worse than the baseline. In the context of this task, safety constraints are violated if either the agent takes longer to reach the goal, or it steps into more number of pits compared to $\pib$.
\end{itemize}

We test on different combinations of user preference $(\bml)$ and baseline's quality $(\rho)$ on 100 randomly generated CMDPs, where $\lambda_i \in \{0, 1\}$, $\rho \in \{0.1, 0.4, 0.7, 0.9\}$ and 
$|D| \in \{ 10, 50, 500, 2000\}$.
We evaluate under two settings: 
(i) we use a fixed set of parameters across different $(\bml, \rho)$ combinations, where we run \ref{eq:s-opt} with $\epsilon \in \{0.01, 0.1, 1.0\}$ and \ref{eq:h-opt} with Doubly Robust IS estimator \citep{jiang2015doubly} and Student’s t-test concentration inequality; 
(ii) we treat them as hyper-parameters that can be optimized for a particular $(\bml,\rho)$ combination. The best hyper-parameters are tuned in a single environment instance and then they are used to benchmark the results on 100 random CMDPs. 

\begin{figure}[t]
\centering
\begin{subfigure}[b]{0.7\textwidth}
    \includegraphics[width=1\textwidth]{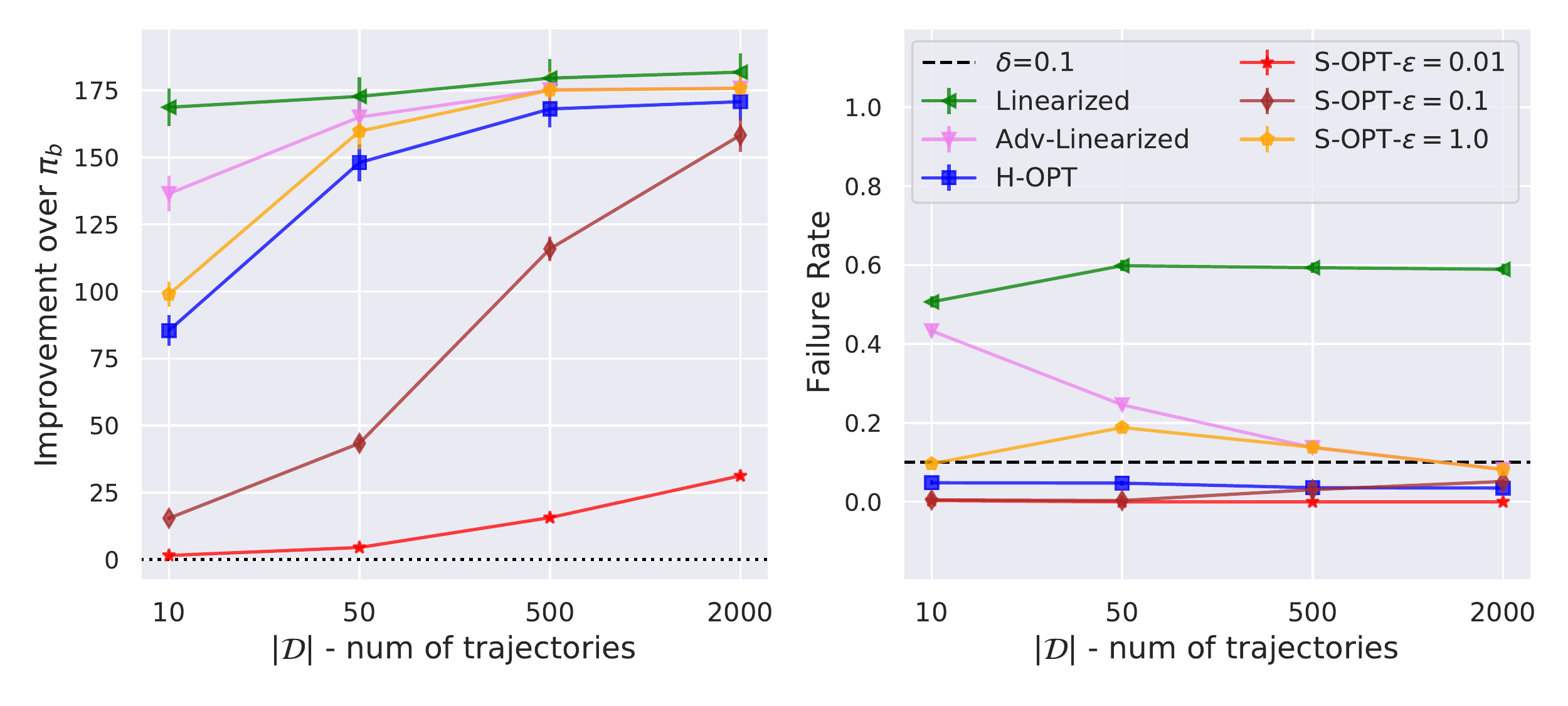}
    \caption{Fixed hyper-parameters}
    \label{fig:delta-params-mean} 
\end{subfigure}
\hfill
\begin{subfigure}[b]{0.7\textwidth}
    \includegraphics[width=1\textwidth]{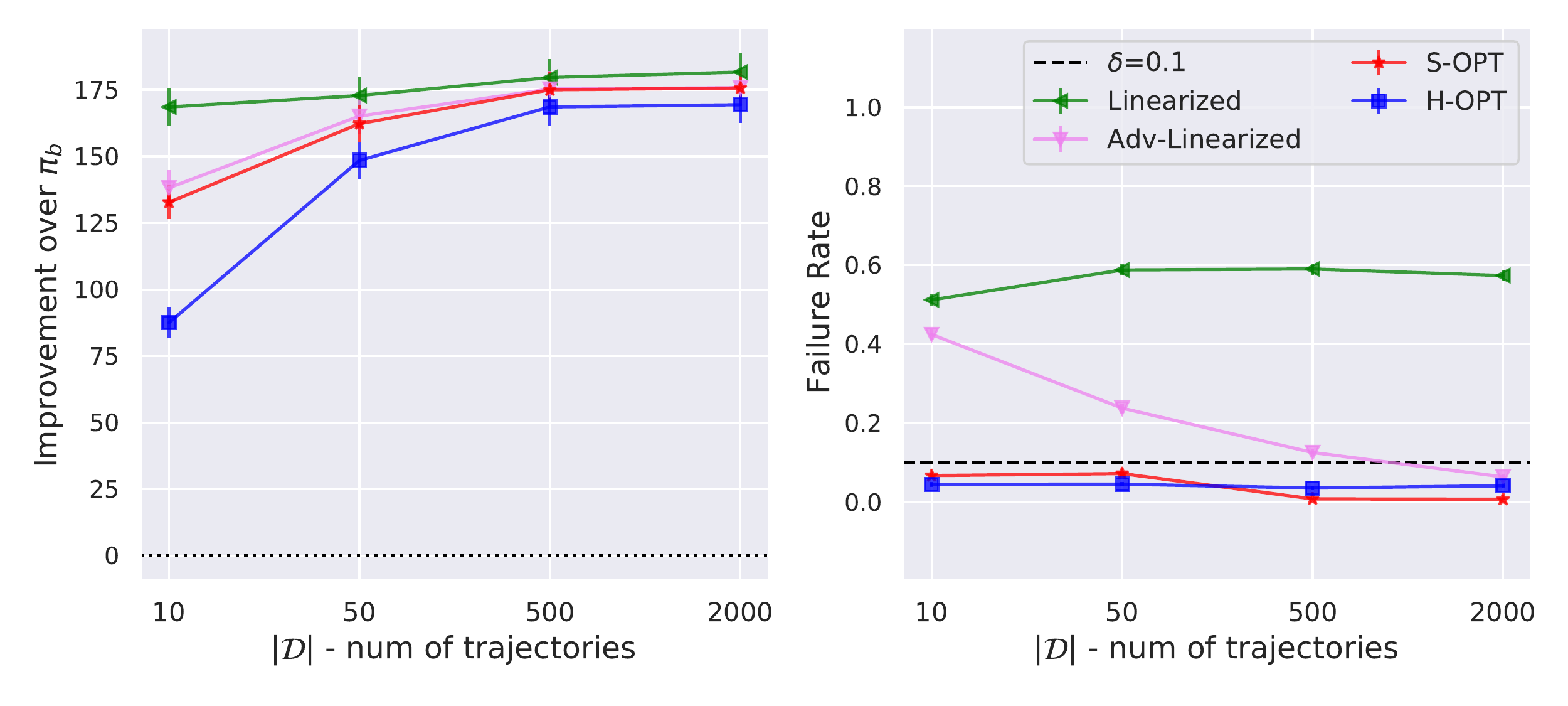}
    \caption{Optimized hyper-parameters.}
    \label{fig:best-params-mean}
\end{subfigure}
\caption[]{
\small
Results on 100 random CMDPs for different $\bml$ and $\rho$ combinations with $\delta=0.1$. The different agents are represented by different markers and colored lines. Each point on the plot denotes the mean (with standard error bars) for 12 different $\bml,\rho$ combinations for the 100 randomly generated CMDPs (1200 datapoints). 
The x-axis denotes the amount of data the agents were trained on. 
The y-axis for left subplot in each sub-figure represents the improvement over baseline and the right subplot denotes the failure rate. The dotted black line in the right subplots represents the high-confidence parameter $\delta=0.1$.
\Cref{fig:delta-params-mean} denotes when the hyper-parameters are fixed $\epsilon=\{0.01, 0.1, 1.0\}$ and $\IS=$ Doubly Robust (DR) estimator with student's t-test concentration inequality. 
\Cref{fig:best-params-mean} is the version with tuned hyper-parameters for each combination.
\label{fig:cmdp-combined-results}}
\vskip -0.1in
\end{figure}

\textbf{Results:} 
The mean results with fixed parameters and $\delta=0.1$ can be found in \Cref{fig:delta-params-mean}. 
The high failure rate of Linearized baseline, regardless of the size of the dataset, is expected as it optimizes the scalarized reward directly and is agnostic of the individual rewards. Adv-Linearized performs better, but in the low data-regime, we see a high failure rate that eventually decreases as the size of dataset increases. This is expected because with more data, more reliable advantage functions estimates are calculated that are representative of the underlying CMDP. 
Compared to the baselines, both \ref{eq:s-opt} and \ref{eq:h-opt} maintain a failure rate below the required confidence parameter $\delta$, regardless of the amount of data.
Also, as the size of dataset increases, we see an increase in improvement over $\pib$, that makes sense as the methods only deviate from baseline when they are sure of the performance guarantees. We expect \ref{eq:s-opt} to violate the constraints with increasing value of $\epsilon$, as it relaxes the constraint on the policy-class (\Cref{eq:spibb-policy-constraint}) and leads to a looser guarantee on performance. This again is reflected in our experiments where \ref{eq:s-opt} with $\epsilon=1.0$ has a higher failure-rate than $\epsilon=0.1$.
We observed similar trends for different $\delta$ values.
A more detailed plot corresponding to different $\bml$ and $\rho$ combinations as well as results for a riskier value of $\delta=0.9$ are given in \Cref{app:cmdp-fixed-param-results}.

The results with optimized hyper-parameters can be found in \Cref{fig:best-params-mean}.
We notice that when the $\epsilon$ parameter is tuned properly, \ref{eq:s-opt} has better performance in terms of improvement over $\pib$ for the same amount of samples when compared to \ref{eq:h-opt}, while still ensuring the failure rate is less than $\delta$. These observations are consistent with the results in the single-objective setting in the original SPIBB works~\citep{laroche2017safe, nadjahi2019safe}.
The general trends and observations from the fixed-parameter case are also valid here.  Additional details, including results for $\bml, \rho$ combinations, hyper-parameters considered and qualitative analysis can be found in \Cref{app:cmdp-best-param-results}. 

We also compare our methods against \cite{le2019batch} in \Cref{app:lag-baseline}. We show the advantage of our approach over \cite{le2019batch}, particularly in the low-data regime, where our methods can improve over the baseline policy while ensuring a low failure rate. 
This makes sense as the method in \cite{le2019batch} relies on the concentrability coefficient which can be arbitrarily high in the low data setting, and therefore their performance guarantees do not hold anymore.
We also provide experiments on the scalability of methods with the number of objectives $d$ in \Cref{app:cmdp-scaling-experiments}.

%% file: doc/sections/real_world_experiments.tex

\section{Real-world experiment}
\label{sec:sepsis-experiments}

In order to validate the applicability of our methods on a real-world  task, 
we consider recent works on 
sepsis management via RL, where we only have access to a pre-collected patient dataset and goal is to recommend treatment strategies for patients with sepsis in the ICU \citep{komorowski2018artificial,tang2020clinician}.
Sepsis is defined as a life-threatening organ dysfunction caused by a dysregulated host response to an infection \citep{singer2016third}.
The main treatment method of sepsis involves the repeated administration of intravenous (IV) fluids and vasopressors, but how to manage their appropriate doses at the patient level is still a key clinical challenge \citep{rhodes2017surviving}. 

%

The problem is safety-critical as our methods need to be cautious about using the data that was possibly collected under unobservable confounders and that can lead to biased model estimates. 
For instance, a study by \citet{ji2020trajectory} of the model used in \citet{komorowski2018artificial} found that the learned model suggests clinically implausible behavior in the form of unexpectedly aggressive treatments.
We show that our methodology can be applied here to prevent such behavior that results from small sample sizes. We propose to do so by incorporating safety constraints to prevent recommending the treatment decisions that were never or rarely performed in the dataset. 




\textbf{Data and MDP Construction:} 
We use the publicly available ICU dataset MIMIC-III  \citep{johnson2016mimic}, with the setup described by \citet{komorowski2018artificial, tang2020clinician} and build on top of their data pre-processing and MDP construction methodology.\footnote{A caveat here is regarding the underlying assumption that the MDP construction methodology by \citet{komorowski2018artificial, tang2020clinician} maintains the Markovian property in the discretized state-space.} 
This leaves us with a cohort of 20,954 unique patients. 
The state-space consisting of 48 clinical variables summarizing features like demographics, physiological condition, laboratory values, etc., is discretized using a k-means based clustering algorithm to map the states to 750 clusters.
The actions include administration of IV fluids and vasopressors, which are categorized into 5 dosage bins each, leading to a total of $|\A|=25$. The $\gamma$ is set to $0.99$. The reward is based on patient mortality. The agent gets a reward, $r_0$, of $\pm 100$ at the end of the episode based on the survival of the patient. More details can be found in \Cref{app:sepsis-dataset}.

In the original work, the rare state-actions taken by the clinicians (state-action pairs observed infrequently in the training set) are removed from the dataset. Instead of removing them,
we define an additional reward, $r_1$, based on the rarity of the state-action pair. We define rare state-action pairs to be those that are taken less than 10 times throughout training dataset, and the agent gets a reward of $-10$ for every such rare state-action taken, i.e.,  $r_1(x,a) = -10.0 \text{ if } \texttt{Count}(x,a) < 10$.
The final task objective then becomes to suggest treatments that handles the trade-off between prioritizing improving the survival vs prioritizing commonly used treatment decisions.

\textbf{Evaluation:} 
We compare our approach with the same baselines from \Cref{sec:synthetic-experiments} on different $\bml$ combinations.  
We run our methods for 10 runs with different random seeds, where for each run the cohort dataset was split into train/valid/test sets in the ratios of 0.7/0.1/0.2.
We evaluate the performance of the solution policies returned by different methods on the test sets using two different OPE methods, Doubly Robust (DR) \citep{jiang2015dependence} and Weighted Doubly Robust (WDR) \citep{thomas2016data}.
We acknowledge that these methods are a proxy of the actual performance of the deployed policies. Hence, these results should not be misinterpreted as us claiming that the policies returned by our methods are now ready to be used in the ICU. 


\begin{table*}[ht!]
\centering
\caption{Performance of various methods using DR and WDR estimators with mean and standard deviation on 10 random splits of the cohort dataset. The red cells denote the corresponding safety constraint violation, i.e, either $\mathcal{J}_{0}^{\pi} < \mathcal{J}_{0}^{\pib}$ or $-\mathcal{J}_{1}^{\pi} > -\mathcal{J}_{1}^{\pib}$.}
\label{table:sepsis-best-results}
\begin{adjustbox}{max width=1\textwidth,center}
\begin{tabular}{cccccc}
\toprule
\multicolumn{1}{c}{User preference $(\bml)$} & \multicolumn{1}{c}{Policy} & \multicolumn{2}{c}{Survival return ($\mathcal{J}_0$)} & \multicolumn{2}{c}{Rare-treatment return ($- \mathcal{J}_1$)} \\
\hline
& & DR & WDR & DR & WDR  \\  \cline{3-6}
& Clinician's ($\pib$) & 64.78 $\pm$ 0.90 & 64.78 $\pm$ 0.90          & 13.58 $\pm$ 0.19 & 13.58 $\pm$ 0.19  \\
\midrule 
\multirow{4}{*}{$[\lambda_0=1, \lambda_1 = 0]$} 
& Linearized & 97.68 $\pm$ 0.22 & 97.58 $\pm$ 0.20   & \textcolor{red}{27.64 $\pm$ 1.11 }& \textcolor{red}{27.84 $\pm$ 1.09 } \\ 
& Adv-Linearized  & 91.62 $\pm$ 0.46 & 92.68 $\pm$ 0.23   & \textcolor{red}{15.18 $\pm$ 0.59 }& 13.56 $\pm$ 0.42 \\
& \ref{eq:s-opt}   & 66.11 $\pm$ 0.87 & 66.05 $\pm$ 0.86   & 13.42 $\pm$ 0.20 & 13.46 $\pm$ 0.20   \\
& \ref{eq:h-opt} & 65.95 $\pm$ 0.00 & 65.95 $\pm$ 0.00   & 13.37 $\pm$ 0.00 & 13.37 $\pm$ 0.00  \\
\midrule 
\multirow{4}{*}{$[\lambda_0=1, \lambda_1 = 1]$}
& Linearized & 87.17 $\pm$ 0.48 & 89.11 $\pm$ 0.37   & 2.41 $\pm$ 0.47 & 1.52 $\pm$ 0.41\\
& Adv-Linearized  & 86.77 $\pm$ 0.49 & 88.58 $\pm$ 0.25   & 2.53 $\pm$ 0.50 & 1.57 $\pm$ 0.43  \\
& \ref{eq:s-opt}  & 86.77 $\pm$ 0.49 & 88.58 $\pm$ 0.25   & 2.53 $\pm$ 0.50 & 1.57 $\pm$ 0.43   \\
& \ref{eq:h-opt} & 86.37 $\pm$ 0.00 & 88.03 $\pm$ 0.00   & 2.58 $\pm$ 0.00 & 1.43 $\pm$ 0.00  \\
\midrule 
\multirow{4}{*}{$[\lambda_0=0, \lambda_1 = 0]$}
& Linearized & \textcolor{red}{-89.39 $\pm$ 0.43} & \textcolor{red}{-90.90 $\pm$ 0.29 }  & \textcolor{red}{22.99 $\pm$ 0.40 }& \textcolor{red}{22.81 $\pm$ 0.30 }  \\ 
& Adv-Linearized  & \textcolor{red}{60.27 $\pm$ 0.49} & \textcolor{red}{61.44 $\pm$ 0.85 }  & \textcolor{red}{18.40 $\pm$ 0.27 }& \textcolor{red}{15.36 $\pm$ 0.58 }  \\
& \ref{eq:s-opt} & 67.73 $\pm$ 0.82 & 67.22 $\pm$ 0.88   & 13.24 $\pm$ 0.24 & 13.55 $\pm$ 0.33  \\
& \ref{eq:h-opt} & 65.95 $\pm$ 0.00 & 65.95 $\pm$ 0.00   & 13.37 $\pm$ 0.00 & 13.37 $\pm$ 0.00  \\
\midrule 
\multirow{4}{*}{$[\lambda_0=0, \lambda_1 = 1]$}
& Linearized & \textcolor{red}{58.27 $\pm$ 2.18} & \textcolor{red}{60.52 $\pm$ 2.07 }  & 0.04 $\pm$ 0.03 & 0.02 $\pm$ 0.01  \\ 
& Adv-Linearized  & 76.05 $\pm$ 0.65 & 76.85 $\pm$ 0.72   & 0.07 $\pm$ 0.05 & 0.04 $\pm$ 0.03  \\
& \ref{eq:s-opt}  & 76.07 $\pm$ 0.65 & 76.87 $\pm$ 0.73   & 0.07 $\pm$ 0.05 & 0.04 $\pm$ 0.03  \\
& \ref{eq:h-opt} & 76.54 $\pm$ 0.00 & 77.55 $\pm$ 0.00   & 0.09 $\pm$ 0.00 & 0.05 $\pm$ 0.00  \\
\bottomrule 
\end{tabular}
\end{adjustbox}
\vskip -0.1in
\end{table*}

\textbf{Results:}
We refer to the return associated with the mortality reward ($r_0$) as survival return ($\mathcal{J}_{0}$), and the negative return associated with rare state-action reward ($r_1$) as rare-treatment return ($- \mathcal{J}_{1}$). Higher survival return implies more successful discharges, and lower rare-treatment return implies more adherence to common practice treatment decisions.
We present the results on survival and rare-treatment returns 
in \Cref{table:sepsis-best-results}. 
As expected, we observe both the Linearized and Adv-Linearized baselines violates constraints across different $\bml$, whereas \ref{eq:s-opt} and \ref{eq:h-opt} are able to respect the safety constraints irrespective of the $\bml$.\footnote{In \Cref{table:sepsis-best-results}, $\bml =[1,1]$ represents a rare case of reward scalarization that allows all the methods to find a good solution policy that satisfies the constraints.  In general, it is difficult to find such scalarization parameters as seen in synthetic experiments (\Cref{app:cmdp-fixed-param-results}).}
The validation set was used to tune the hyper-parameters, and we report how the performance varies with different hyper-parameters in \Cref{app:sepsis-hyperparams}.

\textbf{Qualitative Analysis:} 
We conclude with a qualitative analysis of the policies returned from our setting and the traditional RL approach of maximizing just the survival return. 
%
\citet{ji2020trajectory} found that the RL-policies for sepsis-management task usually end up recommending aggressive treatments, particularly high vasopressor doses for states where the common practice 
(according to most frequent action chosen by the clinician for that state) 
is to give no vasopressors at all. The common practice involves giving zero vasopressors for 722 of the 750 states. However, the policy returned by the traditional single-objective RL baseline recommends vasopressors in 562 (77.84\%) of those 722 states, with 295 of those recommendations being large doses
(upper 50th percentile of nonzero amounts  or $>0.2$ $\mu$g/kg/min). 
We compare these statistics for two of the policies returned by MO-SPIBB that deviate the most from $\pib$. 
The policy returned by \ref{eq:s-opt} ($\bml=[1,1]$) recommends vasopressors in only 93 of those states (12.88 \%), with 47 of those recommendations belonging to high dosages. The other policy, \ref{eq:s-opt} ($\bml=[0,1]$), recommends vasopressors in 134 (18.56 \%) of those states and 70 of those recommendations fall in large dosages.
Therefore, the policies returned by our approach, even when they deviate from the baseline, are less aggressive in recommending rare treatments. 
In \Cref{app:sepsis-qual-analysis}, we present an additional qualitative analysis that demonstrates our methods recommend lesser rare-action treatments than the traditional single-objective RL approach.

An argument can be made against the case when all rare state-action pairs are removed from the training data itself. This will ensure that any learned policy will have near 0 rare-treatment return. However, it is not always clear how to define the cut-off criteria for rare-actions, and it might be possible that some of these rare state-action pairs are actually crucial for finding a better policy. 
For instance, we did an experiment where we assigned state-actions pairs with frequency $<100$ to be rare state-action pairs and filtered those from the training set. The clinician's performance on the test set using a DR estimator for survival return is 65.95. 
In this case, the traditional single-objective RL baseline gives the survival return of 11.26, 
which shows that removing such transitions from the dataset actually hampers the solution quality. Our approach of assigning a separate reward for rare state-action pairs is able to find a solution with a survival return of 86.75 
even in this scenario.

%% file: doc/sections/discussion.tex
\section{Conclusion}
\label{sec:conclusion}

We present a new Seldonian RL algorithm that takes the user preference based scalarization into account while ensuring the solution policy performs reliably in context to the baseline policy across all objectives. 
On both synthetic and real-world tasks, we show that the proposed approach can improve the policy while ensuring the safety constraints are respected. 

Our setting can accommodate any general form of scalarizations (e.g. non-linear or convex) as well as objectives (such as fairness), making it applicable to a wide variety of real-world tasks. 
The only assumption we made is regarding the dataset being collected under a single known baseline policy. An exciting line of future work can be to relax this assumption and consider the scenario where the dataset comes from a variety of unknown policies with different qualities. 
We did not make any claims about the optimality of the solutions as often optimality and safety are contradicting objectives. It is not clear how (and if) one can make claims about optimality in the offline setting without bringing in additional unrealistic assumptions (\Cref{sec:related-work}). 
The extension to infinite MDPs and the function-approximation setting is also left for future work. 
It is important to note that when it comes to practical application, it is not unusual for continuous domains to be discretized to enable better interpretability, especially when interactions with humans are necessary. If the Markovian property is valid in the discretized space, SPIBB-based guarantees will also hold true.

%% file: doc/sections/acknowledgements.tex
\section{Acknowledgements}

The authors would like to thank NSERC (Natural Sciences and Engineering Research Council), IVADO (Institut de valorisation des données) and CIFAR (Canadian Institute for Advanced Research) for funding to McGill in support of this research. Philip S. Thomas was funded in part by NSF award \#2018372. 
The computational component of this research was enabled in part by support provided by Calcul Québec (\url{www.calculquebec.ca}), Compute Canada (\url{www.computecanada.ca}) and Mila's IDT team. 

We would also like to thank Emmanuel Bengio and Koustuv Sinha for many helpful
discussions about the work, and the anonymous reviewers for providing constructive feedback.

%% file: doc/appendix/naive_obj_construction.tex
\section{Scalarized safety constraints}
\label{app:naive-construction}

Instead of having constraints of the form in \Cref{eq:general-safety-constraints}, it is possible to define the constraints in terms of the scalarized objective directly, i.e., 
\begin{align*}
    \mathbb{P}\left(    \sum_{k\in[d]} \lambda_i \J{\pi}{i}{\mstar}  - \sum_{k\in[d]} \lambda_i \J{\pib}{i}{\mstar} > - \zeta \Big| \D \right) > 1 -\delta.
\end{align*}

Without loss of generality, if we assume there are only two objective ($d=2$), then satisfying the above constraint implies:
\begin{align*}
    \lambda_0 \left(\J{\pi}{0}{\mstar} - \J{\pib}{0}{\mstar} \right) + \lambda_1 \left(\J{\pi}{1}{\mstar} - \J{\pib}{1}{\mstar}  \right) \geq 0 .
\end{align*}
Consider a scenario where the solution policy performs poorly w.r.t. the second objective, i.e, $\lambda_1 (\J{\pi}{1}{\mstar}  - \J{\pib}{1}{\mstar}) < 0$, however the the improvement in the first objective is very large $\lambda_0 (\J{\pi}{0}{\mstar} - \J{\pib}{0}{\mstar}) >> 0$. In this case, even though the linearized cumulative constraint regarding the performance improvement is being satisfied, it fails to guarantee the improvement across each individual objectives.




%% file: doc/appendix/spibb_theorems.tex
\section{SPIBB - Additional details}
\label{app:spibb-additional-details}  
\subsection{Concentration Bounds} 
\label{app:error_bounds}

The difference between an estimated parameter and the true one can be bounded using concentration bounds (or equivalently, Hoeffding's inequality) applied to the state-action counts $n_{\mathcal{D}}(x,a)$ in dataset $\mathcal{D}$~\citep{petrik2016safe, laroche2017safe}. Specifically, the following inequalities hold with probability at least $1-\delta = 1 - \delta' - \delta''$ for any state-action pair $(x,a) \in \mathcal{X} \times \mathcal{A}$:
\begin{align}
	\lVert p^\star(\cdot|x,a)-\hat{p}(\cdot|x,a)\rVert_1 &\leq e_p(x,a),\\
	\forall k\in[d], \lvert r^\star_k(x,a)-\hat{r}_k(x,a)\rvert &\leq e_r(x,a)r_{\mytop} \label{eq:bound-Q-2}
\end{align}
where: 
\begin{align}
    e_p(x,a) &:= \sqrt{\cfrac{2}{n_{\mathcal{D}}(x,a)}\log\cfrac{2|\mathcal{X}||\mathcal{A}|2^{|\mathcal{X}|}}{\delta'}} \label{eq:error-function-P-2} \\
    e_r(x,a) &:= \sqrt{\cfrac{2}{n_{\mathcal{D}}(x,a)}\log\cfrac{2|\mathcal{X}||\mathcal{A}|d}{\delta''}}. \label{eq:error-function-Q-2}
\end{align}

    The two inequalities can be proved similarly to \citep[Proposition 9]{petrik2016safe}. We only detail the proof for \eqref{eq:bound-Q-2}: for any $(x,a) \in \mathcal{X} \times \mathcal{A}$, and from the two-sided Hoeffding's inequality, 
    \begin{align*}
        &\mathbb{P} \left( \forall (x,a), \big\lvert r^\star_k(x,a) -\hat{r}_k(x,a) \big\rvert > e_r(x,a)r_{\mytop} \right) 
        \\
        &\qquad\qquad = \mathbb{P} \left( \forall (x,a), \frac{\big\lvert r^\star_k(x,a)-\hat{r}_k(x,a) \big\rvert}{2 V_{max}} > \sqrt{\frac{1}{2 n_\mathcal{D}(x,a)} \log \frac{2 |\mathcal{X}||\mathcal{A}|d}{\delta''}} \right) \\
        &\qquad\qquad \leq 2 \exp \left( -2 n_\mathcal{D}(x,a) \frac{1}{2 n_\mathcal{D}(x,a)} \log \frac{2 |\mathcal{X}||\mathcal{A}|}{\delta''}  \right) \\
        &\qquad\qquad \leq \frac{\delta''}{| \mathcal{X} | | \mathcal{A} | d}
    \end{align*}
    
    By summing all $|\mathcal{X}| |\mathcal{A}|d $ state-action-reward tuples error probabilities lower than $\frac{\delta''}{| \mathcal{X} | | \mathcal{A} |d }$, we obtain \eqref{eq:bound-Q-2}. If we choose $e(x,a)=e_p(x,a)=e_r(x,a)$, we get that:
    \begin{align}
        \cfrac{2}{n_{\mathcal{D}}(x,a)}\log\cfrac{2|\mathcal{X}||\mathcal{A}|2^{|\mathcal{X}|}}{\delta'} &= \cfrac{2}{n_{\mathcal{D}}(x,a)}\log\cfrac{2|\mathcal{X}||\mathcal{A}|d}{\delta''} \\
        \cfrac{2^{|\mathcal{X}|}}{\delta'} &= \cfrac{d}{\delta''} \\
        \delta'' &= d\delta' 2^{-|\mathcal{X}|}
    \end{align}
    
    It means that $\delta = \delta' + \delta'' = \delta'(1+d2^{-|\mathcal{X}|})$. The cost in terms of approximation is therefore linear with a very small slope, inside square root of log, which means that it will basically have an insignificant impact on the concentration bound.

\subsection{Need of advantageous constraints}
\label{app:spibb-need-of-advatangeous}

\begin{prop}
\label{prop:mo-sipbb-advantageous}
The advantageous constraints in \ref{eq:s-opt} ensure that performance constraints w.r.t. the individual returns are respected in $\mhat$, i.e., $ \forall k \in [d],\; \J{\pi}{k}{\mhat} - \J{\pib}{k}{\mhat} \ge 0$.
\end{prop}

\begin{proof}
For the $k$\textsuperscript{th} reward function, we can estimate the advantage function in an MDP $m$ as:
\begin{align*}
    \adv{\pib}{k}{m}{x,a} = \qval{\pib}{k}{m}{x,a} - \val{\pib}{k}{m}{x}
\end{align*}
Similarly, let $\rho^{\pi}_{m}(x)$ denote the normalized discounted future state distribution:
\begin{align*}
    \rho^{\pi}_{m}(x) &= (1-\gamma)\sum_{t=0}^{\infty} \gamma^t \mathbb{P}(X_t=x | \pi, X_0=x_0),
\end{align*}
where $X_{t} \sim p(\cdot | X_{t-1}, A_{t-1}), A_{t-1} \sim \pi(\cdot|X_{t-1})$.
From Performance Difference Lemma \citep{kakade2002approximately}, we have the following result:
\begin{align}
    \label{eq:spibb-prop-adv}
    \J{\pi}{k}{\mhat} - \J{\pib}{k}{\mhat} &= \sum_{x \in \X} \rho^{\pi}_{\mhat}(x) \underbrace{\sum_{a \in \A} \pi(a|x) \adv{\pib}{k}{\mhat}{x,a}}_{\text{advantage constraint}}
\end{align}

The first term in the above equation $\rho^{\pi}_{\mhat}(x) \ge 0$ for any $x \in \X$. The second term is the advantage constraint in the construction of \ref{eq:s-opt}.
Therefore, any solution of \ref{eq:s-opt} satisfies $\sum_{a \in \A} \pi(a|x) \adv{\pib}{k}{\mhat}{x,a} \ge 0, \forall x \in \X$.

As both the terms in \Cref{eq:spibb-prop-adv} are $\ge 0 \; \forall x \in \X$, this implies $ \J{\pi}{k}{\mhat} - \J{\pib}{k}{\mhat} \ge 0$.

\end{proof}

\subsection{MO-SPIBB Results}
\label{app:mo-spibb-prop}

Using the results from \Cref{app:error_bounds} and \Cref{app:spibb-need-of-advatangeous}, we can directly apply the Soft-SPIBB theorems to individual objectives in \ref{eq:s-opt}. For instance, we get the following result about 1-step policy improvement guarantees directly from Theorem 1 of Soft-SPIBB:

\begin{prop}
The policy $\pi$ returned from solving the \ref{eq:s-opt} satisfies the following property in every state $x$ with probability at least $(1 - \delta)$:
\begin{align}
    \forall k \in [d], \val{\pi}{k}{\mopt}{x} - \val{\pib}{k}{\mopt}{x} \geq -\frac{\epsilon v_{\text{max}}}{1-\gamma},
\end{align}
where $v_{\text{max}} \le \frac{\rmax}{1-\gamma}$ is the maximum of the value function.
\end{prop}

\begin{proof}
We will show the policy returned by \ref{eq:s-opt} satisfies both the properties required for applying the Theorem 1 of Soft-SPIBB:

\begin{itemize}
    \item $\pi$ is $(\pib, \epsilon, e)$-constrained: This is equivalent to $\sum_{a \in \A} e(x,a)\ |\pi(a|x) - \pib(a|x)| \leq \epsilon$, that is true by construction.
    
    \item $\pib$-advantageous in $\mhat$: For $k$\textsuperscript{th} reward function, this is equivalent to $\J{\pi}{k}{\mhat} - \J{\pib}{k}{\mhat} \ge 0$, which is also true from construction.
\end{itemize}
From there, the exact statement of Theorem 1 can be applied directly to get the above result.
\end{proof}

%% file: doc/appendix/hcpi_details.tex
\section{HCPI - Additional details}
\label{app:hcpi-details}

\paragraph{Concentration Inequalities:} We experimented with the following concentration inequalities \citep{thomas2015highImprovement}:
\begin{itemize}[leftmargin=*]
    \item Extension of Empirical Bernstein \citep{maurer2009empirical}: This is the extension of
    Maurer \& Pontil's empirical Bernstein (MPeB) inequality. From Theorem 1 of \cite{thomas2015highEvaluation}: Let $X_1, \dots, X_n$ denote $n$ independent real-valued random variables, such that for each $i \in \{1,\dots,n\}$, we have $\pr(0 \le X_i) = 1,  \E[X_i] \le \mu$, and some fixed real-valued threshold $c_i > 0$. Let $\delta > 0$ and $Y_i = \min\{X_i, c_i\}$, then with probability at least $(1-\delta)$:
    \begin{align}
        \mu &\ge \sum_{i=1}^{n}\left(\frac{1}{c_i}\right)^{-1} \sum_{i=1}^{n} \frac{Y_i}{c_i} - \sum_{i=1}^{n}\left(\frac{1}{c_i}\right)^{-1} \frac{7n \ln(2/\delta)}{3n-1} - 
        \sum_{i=1}^{n}\left(\frac{1}{c_i}\right)^{-1} \sqrt{\frac{\ln(2/\delta)}{n-1} \sum_{i,j=1}^{n}\left(\frac{Y_i}{c_i} - \frac{Y_j}{c_j}\right)^2}.
    \end{align}
    In context of this paper, for the $k$\textsuperscript{th} reward function, $X_i$ denotes the $\IS$ estimated return for that trajectory, i.e., $\IS_k(\tau_i,\pi_t, \pib)$. Here, $c_i$ is a hyper-parameter that needs to be tuned. In \cite{thomas2015highImprovement}, a fixed value of $c$ is used for all $c_i$. 
    
    \item Student's t-test \citep{walpole1993probability}: This is an approximate concentration inequality that is based on the assumption that the mean returns are distributed normally. For $k$\textsuperscript{th} reward, the \Cref{eq:hcope-R-lower-bound} can be written as:
    \begin{align}
    \pr \Big( \J{\pi_t}{k}{\mopt} \ge \IS_{k}(\D, \pi_t, \pib) - \frac{\hat{\sigma}_k}{\sqrt{|\D|}}t_{1-\delta/d, |\D|-1} \Big) \ge 1 - \delta/d, 
    \end{align}
    where $\hat{\sigma}_k$ is the sample standard deviation:
    \begin{align}
    \hat{\sigma}_k &= \sqrt{\frac{1}{|\D|-1} \sum_{i=1}^{|\D|}(\IS(\tau_i, \pi_t, \pib) - \overline{\IS} )^2 },
    \end{align}
    and $\overline{\IS} = \frac{1}{|\D|}\sum_{i=1}^{|\D|} \IS(\tau_i,\pi_t,\pib)$ and $t_{1-\delta/d, |\D|-1}$ is the $100(1-\delta/d)$ percentile of the student t-distribution with $|\D|-1$ degrees of freedom. 
    
\end{itemize}

We experimented with both MPeB Extension (with $c=0.5$) and Student's t-test inequalities and found that the solutions returned by the former to be very conservative. Therefore, we use t-test in all of our experiments. Even though the t-test's assumption (normally distributed returns) is technically false, it's a reasonable assumption due to central limit theorem. 
The consequence is that the failure rate (the chance of deploying an unsafe policy) can, in theory, be higher than desired, though, in practice, that's unlikely.

\paragraph{Regularization:} 
For small problems, \ref{eq:h-opt} can be solved with methods like CMA-ES \citep{hansen2006cma}. 
%
For stochastic policies, 
as the optimization problem in \ref{eq:h-opt} is difficult to solve directly, we need to resort to a regularization based heuristic \citep{thomas2015highImprovement, laroche2017safe}. Let $\pi_t$ denote the solution policy found using $\D_{tr}$ using any of the traditional offline RL methods. A set of candidate policies is built using the baseline policy: $\pi_{\text{Cand}} = \{ (1-\alpha)\pi_t + \alpha \pi_b\}$, where $\alpha \in \set{0.0, 0.1, \dots, 0.9}$ is the regularization hyper-parameter. 
The best performing candidate policy that satisfies the safety-test (the performance constraints based on $\D_s$) is then returned.
If none of the candidate policies satisfy the safety-test, the baseline policy is returned.

For finding $\pi_t$, we experimented with both the Linearized and Adv-Linearized baselines in \Cref{sec:synthetic-experiments} and found that Adv-Linearized worked better (higher improvement over $\pib$ while failure rate $<\delta$). Therefore in our experiments, we first find $\pi_t$ using Adv-Linearized and then regularize it using $\pib$ to build the set of candidate policies $\pi_{\text{Cand}}$.

\paragraph{Safety-guarantees:}
We get the safety guarantees related to \ref{eq:h-opt} directly from \cite{thomas2015highImprovement, Thomas2019}.  The constraints of \ref{eq:h-opt} define the new safety-test that ensures a candidate policy will only be returned if the individual performance guarantees corresponding to each reward function are satisfied. This procedure will only make error in the scenario where the performance constraint related to $k$\textsuperscript{th} is satisfied, i.e, $( \IS_{k}(\D_{s}, \pi, \pib) - \CI_k(\D_s, \delta/d) \geq \mu_k)$,  but in practice the policy is not good enough $(\J{\pi}{k}{\mopt} < \mu_k)$. By transitivity this implies $\J{\pi}{k}{\mopt} < \big( \IS_{k}(\D_{s}, \pi, \pib) - \CI_k(\D_s, \delta/d) \big)$, which from \Cref{eq:hcope-R-lower-bound} we know can only occur with probability at most $\delta/d$. Using the union bound, we know that cumulative probability of the union of any of these $d$ possible scenarios is $\leq \delta$.

\paragraph{Computational cost:} Compared to regular HCPI, there is an increase in computational cost proportional to the number of reward functions $d$. The value and advantage functions estimation cost increases by a factor of $d$: respectively $\mathcal{O}(d|\X|^3)$ and $\mathcal{O}(d|\A||\X|^2)$, the $\IS$ estimation also increases by factor of $d$, and the computational cost for safety-test also increases by $d$: $\mathcal{O}(d|\D|)$.

%% file: doc/appendix/cmdp_extra_details.tex
\section{Additional details for synthetic CMDP experiments}
\label{app:additional-details-for-synthetic-exp}

\subsection{Solving CMDP}
\label{app:cmdp-solver}

Constrained-MDPs~\citep{altman1999constrained} are MDPs with multiple rewards where $r_0$ is the main objective, and $r_1, \dots, r_{n-1}$ are the reward signals that are used to enforce some behavior or constraints. 

Let $\J{\pi}{i}{m}(\mu)$ denote the total expected discount reward under $r_i$ in an MDP $m$, when $\pi$ is followed from an initial state chosen at random from $\mu$, the initial state distribution. 
For some given reals $c_1, \dots, c_n$ (each corresponding to $r_i$), the CMDP optimization problem is to find the policy that maximizes the $\J{\pi}{0}{m}(\mu)$ subject to the constraints $\J{\pi}{i}{\mopt}(\mu) \le c_i$ :
\begin{align}
    \label{eq:cmdp-obj}
    &\max_\pi \J{\pi}{0}{m}(\mu) \\ 
        \quad \text{ s.t. } & \J{\pi}{i}{m}(\mu) \le c_i, \, \forall i \in \{1,\dots,n-1\}. \nonumber
\end{align}
    

The Dual LP based algorithm for solving CMDP is based on the occupation measure w.r.t. the optimal policy $\piopt$. For any policy $\pi$ and initial state $x_0 \sim \mu(\cdot)$, the occupancy measure is described as:
\begin{align*}
    \rho^{\pi}(x,a) &= \E \left[ \sum_{t=0}^{\infty} \gamma^t \mathbbm{1}\{x_t=x, a_t=a\} \Big| x_0, \pi \right], \forall x \in \X, \forall a \in \A. 
\end{align*}
The occupation measure at any state $x \in X$ is defined as $\sum_{a} \rho^{\pi}(x,a)$.  From \citep[Chapter 9]{altman1999constrained}, the problem of finding the optimal policy for a CMDP can be solved by the solving the following LP problem: 
\begin{align*}
    \max_{\rho}  &\quad \sum_{x \in \X, a \in \A} \rho(x,a) r_0(x,a) \\  
    \texttt{s.t.}  &\quad \sum_{x \in \X, a \in \A} \rho(x,a) r_i(x,a) \leq c_i, \; \forall i \in \{1,\dots,n-1\}.
\end{align*}
 As $\rho$ is the occupation measure it also needs to satisfy the following constraints $\forall x \in \X$:
\begin{align*}
    \rho(x,a) &\geq 0,  \quad \forall a \in \A \\
    \sum_{x_p \in \X, a \in \A} \rho(x_p,a) (\mathbbm{1}\{ x_p = x\} - p(x|x_p,a))  &= \mathbbm{1}\{ x=x_0 \}
\end{align*}
The above constraints originate from the conservation of probability mass of a stationary distribution on a Markov process.  The state-action visitations should satisfy the single-step transpose Bellman recurrence relation:
\begin{equation*}
    \rho^{\pi}(x,a) = (1-\gamma) \mu(x) \pi(a|x) + \gamma \cdot p_{T}^{\pi} \rho^{\pi}(x,a), 
\end{equation*}
where transpose policy transition operator $p_{T}^{\pi}$ is a linear operator and is the mathematical transpose (or adjoint) of $p^{\pi}$ in the sense that $<y, p^{\pi}x> = <p_{T}^{\pi} y, x>$ for any $x, y$:
\begin{equation*}
    p_{T}^{\pi} \rho(x,a) \doteq \pi(a|s)\sum_{\tilde{x}, \tilde{a}} p(x|\tilde{x}, \tilde{a}) \rho(\tilde{x}, \tilde{a})
\end{equation*}

In conclusion, the complete dual problem can be written as:
\begin{align}
    \label{eq:cmdp-opt}
    \max_{\rho: \X \times \A \rightarrow \mathbb{R}_{+}}  &\quad \sum_{x \in \X, a \in \A} \rho(x,a) r_0(x,a) \\  
    \texttt{s.t.}  &\quad \sum_{x \in \X, a \in \A} \rho(x,a) r_i(x,a) \leq c_i,  \; \forall i \in \{1,\dots,n-1\}, \nonumber \\
    &\quad \sum_{a} \rho(x,a) = \sum_{\tilde{x},\tilde{a}} p(x|\tilde{x},\tilde{a}) \rho(\tilde{x},\tilde{a}) + \mu(x). \tag{$\forall x \in \X$}
\end{align}

The solution of the above problem $\rho^{\star}$ gives the optimal (stochastic) policy of the form:
\begin{align*}
    \piopt(a|x) &= \frac{\rho^{\star}(x,a)}{\sum_a \rho^{\star}(x,a)} , \forall x \in \X, \forall a \in \A. 
\end{align*}

\subsection{Additional results with fixed hyper-parameters}
\label{app:cmdp-fixed-param-results}

\Cref{fig:delta-0x1-params-grid} gives the individual plots for different $\bml, \rho$ combinations corresponding to the plot in \Cref{fig:delta-params-mean}. This is the fixed parameters setting in \Cref{sec:synthetic-experiments} where the same set of parameters are used across different $\bml, \rho$ combinations. Here, we run \ref{eq:s-opt} with $\epsilon \in \{0.01, 0.1, 1.0\}$ and \ref{eq:h-opt} with Doubly Robust IS estimator \citep{jiang2015doubly} and student's t-test.  
The mean results with $\delta=0.9$ can be found in \Cref{fig:extra-delta-0.9-params-mean}. A more detailed plot containing the $\bml, \rho$ wise breakdown can be found in \Cref{fig:extra-delta-0.9-params-grid}.

\begin{figure*}
  \includegraphics[width=\textwidth]{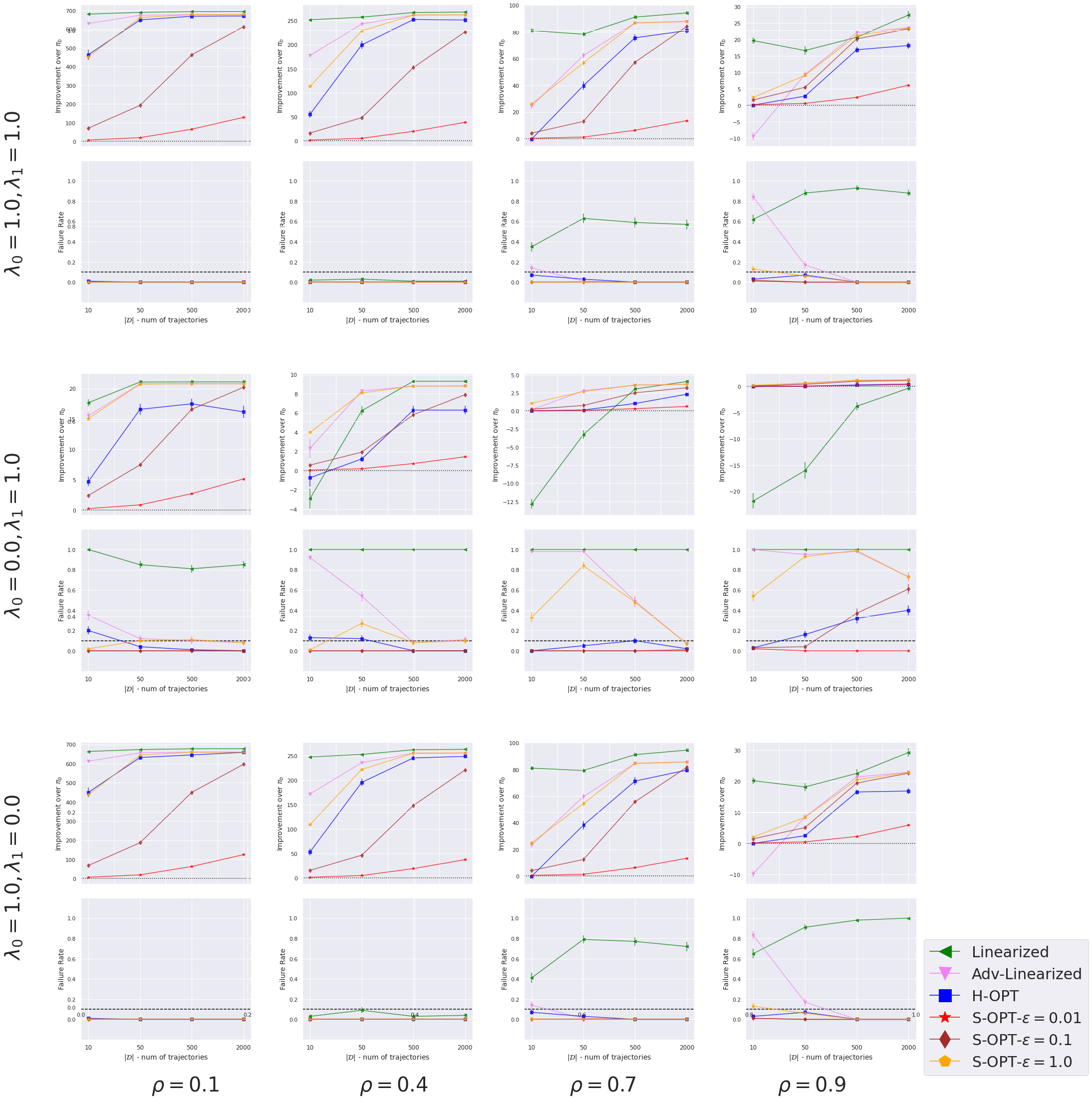}
  \caption{Results on random CMDPs with fixed parameters and $\delta=0.1$. 
  The different agents are represented by different markers and color lines. 
  Each point on the grid, corresponding to a $\bml, \rho$ combination, denotes the mean (with standard error bars) for the 100 randomly generated CMDPs. 
  The x-axis denotes the amount of data the agents were trained on. They y-axis for the top subplot in a grid cell represents the improvement over baseline and the y-axis for bottom subplot in a grid cell denotes the failure rate.
  The dotted black line represents the high-confidence parameter $\delta=0.1$.
  }
  \label{fig:delta-0x1-params-grid}
\end{figure*}

\begin{figure}
    \centering
  \includegraphics[scale=0.4]{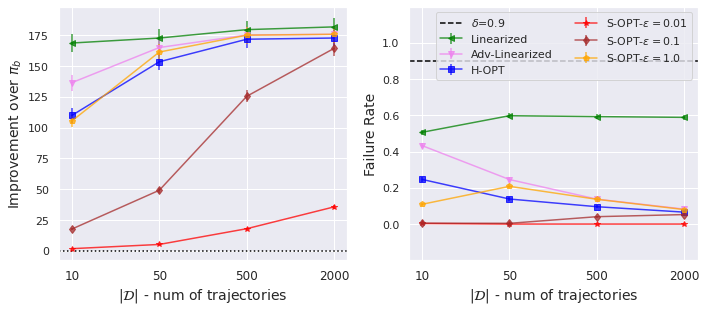}
  \caption{
  Mean results on random CMDPs with fixed parameters and $\delta=0.9$. 
  The different agents are represented by different markers and color lines. 
  Each point on the plot denotes the mean (with standard error bars) for 12 different $\bml,\rho$ combinations for the 100 randomly generated CMDPs (1200 datapoints). 
  The x-axis denotes the amount of data the agents were trained on. They y-axis for the left subplot represents the improvement over baseline and the y-axis for the right subplot in a grid cell denotes the failure rate.
  The dotted black line represents the high-confidence parameter $\delta=0.9$.
  }
  \label{fig:extra-delta-0.9-params-mean}
\end{figure}

\begin{figure*}
  \includegraphics[width=\textwidth]{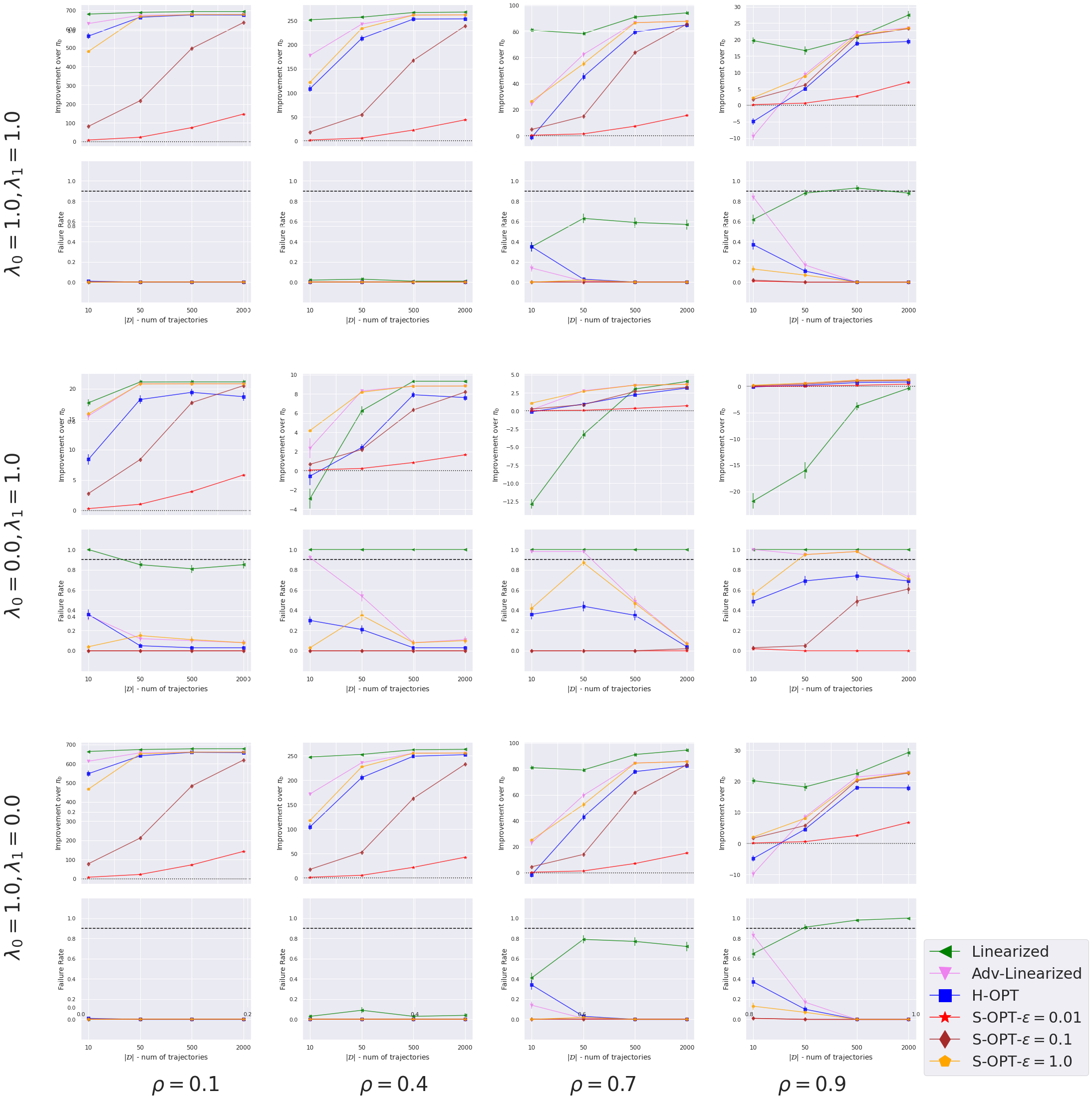}
  \caption{
  Results on random CMDPs with fixed parameters and $\delta=0.9$. 
  The different agents are represented by different markers and color lines. 
  Each point on the grid, corresponding to a $\bml, \rho$ combination, denotes the mean (with standard error bars) for the 100 randomly generated CMDPs. 
  The x-axis denotes the amount of data the agents were trained on. They y-axis for the top subplot in a grid cell represents the improvement over baseline and the y-axis for bottom subplot in a grid cell denotes the failure rate.
  The dotted black line represents the high-confidence parameter $\delta=0.9$.
  }
  \label{fig:extra-delta-0.9-params-grid}
\end{figure*}

\subsection{Additional results with tuned hyper-parameters}
\label{app:cmdp-best-param-results}

\Cref{fig:best-params-grid} gives the individual plots for different $\bml, \rho$ combinations corresponding to the plot in \Cref{fig:best-params-mean}.
The best hyper-parameters are tuned in a single environment and then are used to benchmark the results on 100 random CMDPs. The following procedure is used for selecting the best hyper-parameter candidates: We first generate a random CMDP and run different hyper-parameters on that environment instance. Next, we filter the candidates that violate the safety-constraint in that CMDP instance. From the remaining candidates, we select the one that yields the highest improvement over $\pib$. 

For \ref{eq:s-opt}, we searched for $\epsilon \in \{1e^{-4}, 1e^{-3}, 1e^{-2}, 1e^{-1}, 0.5, 1.0, 2.0, 5.0\}$. For \ref{eq:h-opt}, we used student's t-test with the following $\IS$ estimators: Importance Sampling (IS), Per Decision IS (PDIS), Weighted IS, Weighted PDIS and Doubly Robust (DR) \citep{precup2000eligibility, jiang2015doubly}.

\begin{figure*}
  \includegraphics[width=\textwidth]{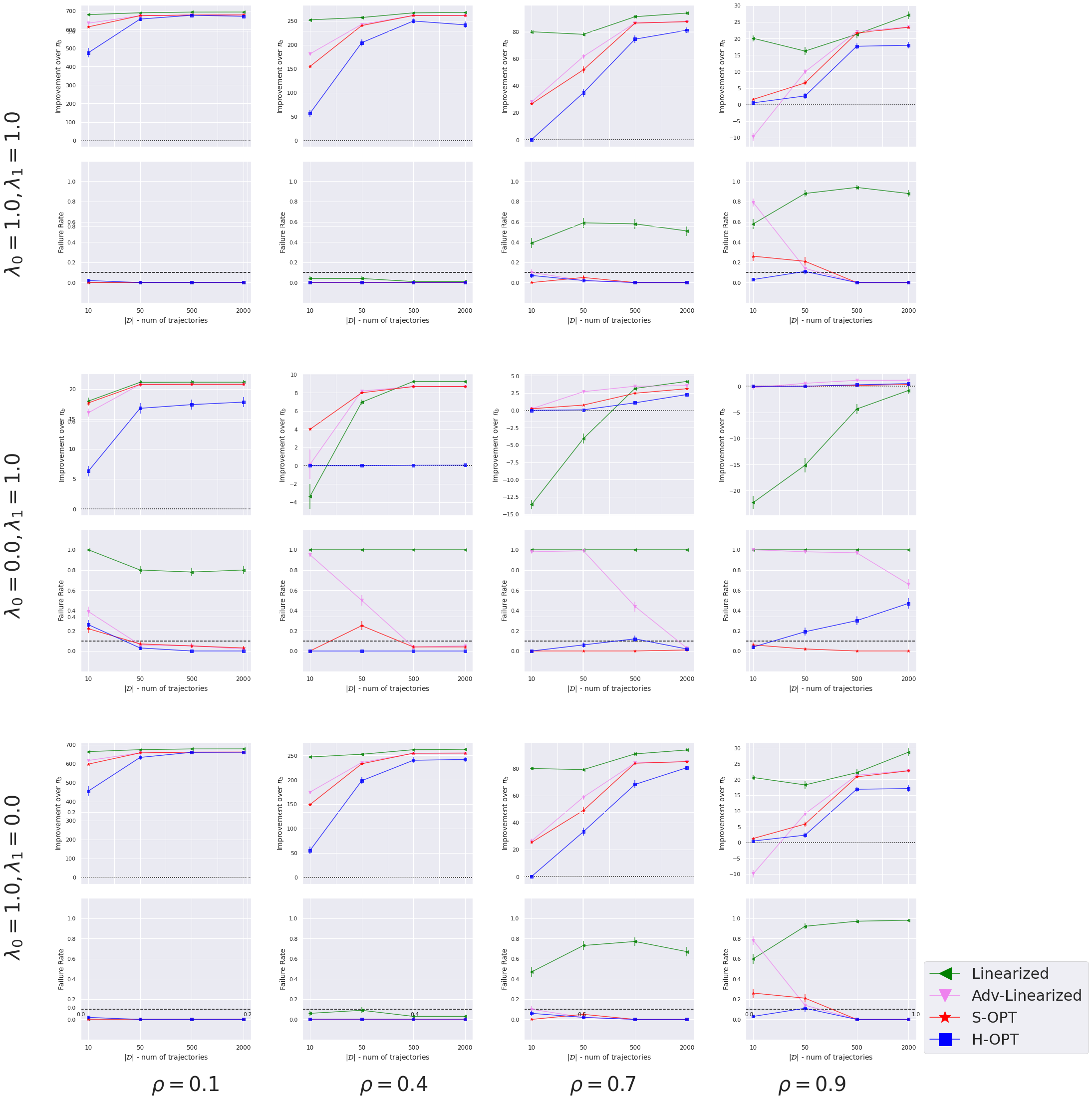}
  \caption{
  Results on 100 random CMDPs for different $\bml, \rho$ combinations with best $\epsilon, \IS$ combination for $\delta=0.1$.
  The different agents are represented by different markers and color lines. 
  Each point on the grid, corresponding to a $\bml, \rho$ combination, denotes the mean (with standard error bars) for the 100 randomly generated CMDPs. 
  The x-axis denotes the amount of data the agents were trained on. They y-axis for the top subplot in a grid cell represents the improvement over baseline and the y-axis for bottom subplot in a grid cell denotes the failure rate.
  The dotted black line represents the high-confidence parameter $\delta=0.1$.
  }
  \label{fig:best-params-grid}
\end{figure*}


We plot the results based on the optimized hyper-parameters for a single CMDP in \Cref{fig:10x10-gridworld-seed-0} . Here, we plot the individual performance w.r.t $r_0$ (goal reward) and $r_1$ (pit reward) for multiple agents along with the baseline's performance.  Instead of working with surrogate measures, we investigate the returns for both $\J{\pi}{r_0}{\mopt}$ and $- \J{\pi}{r_1}{\mopt}$, and see what kind of scenarios lead to violation (all the returns are normalized in $[0,1]$). In \Cref{fig:10x10-gridworld-seed-0}, the intersection of the red and blue lines denotes the performance of the baseline in the true MDP.  As we observed in the mean plots, the Linearized baseline violate most of constraints for all the dataset sizes. The Adv-Linearized baseline violates the constraints mostly for low data settings ($\blacktriangledown$ marker with darker shades). There are more violations for higher values of $\rho$ as the $\pib$ gets better and the task gets tougher. We can observe that both \ref{eq:s-opt} and \ref{eq:h-opt} based agents (denoted by $\star$ and $\blacksquare$ markers) never leave the top-left quadrant and consistently satisfy the constraints. We also observe that the deviation from the origin increases with the increase in dataset size (represented via color of the agent).

\begin{figure*}
  \includegraphics[width=\textwidth]{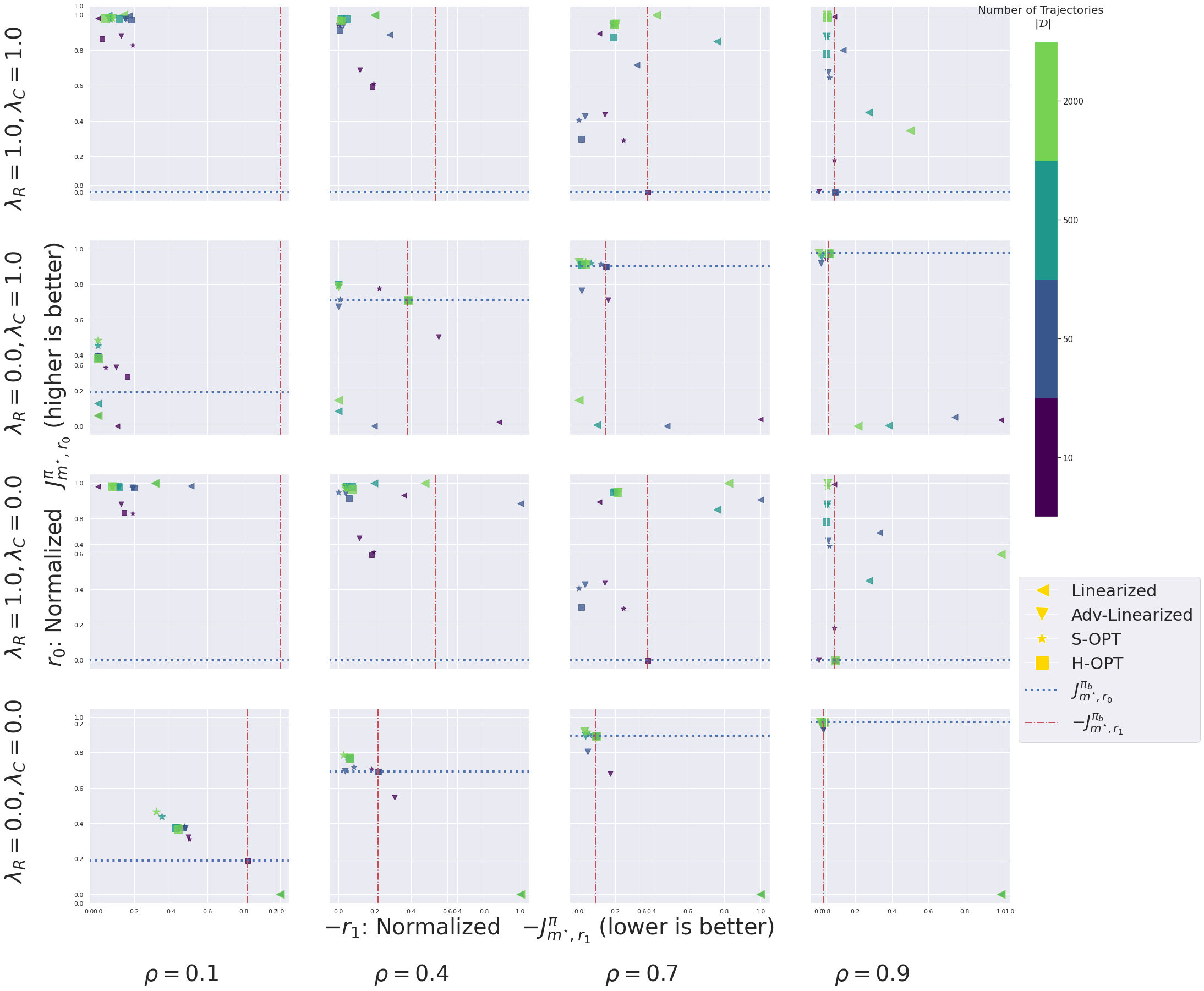}
  \caption{
  Results on a random $10 \times 10$ synthetic CMDP. Each $\bml$ and $\rho$ combination represents a different setting denoted by the corresponding cell in the grid. The different agents are represented by different markers and the color of the marker denotes the amount of data the agent was trained on. 
  The x-axis for individual plots are normalized $- \mathcal{J}^{\pi}_{\mopt, r_1}$ returns (for pits), and y-axis are normalized $\mathcal{J}^{\pi}_{\mopt, r_0}$ returns (for goal).
  The red line denotes the performance of the baseline w.r.t. $- \mathcal{J}^{\pib}_{\mopt, r_1}$, and the blue line for $\mathcal{J}^{\pib}_{\mopt, r_0}$. For each plot in the grid, only the points in the top-left quadrant (defined by baseline's performance via red and blue lines) satisfy the constraint for that task.
  }
  \label{fig:10x10-gridworld-seed-0}
\end{figure*}

\subsection{Comparison with \cite{le2019batch}}
\label{app:lag-baseline}

\begin{figure}[t]
\centering
\begin{subfigure}[b]{1\textwidth}
    \includegraphics[width=1\textwidth]{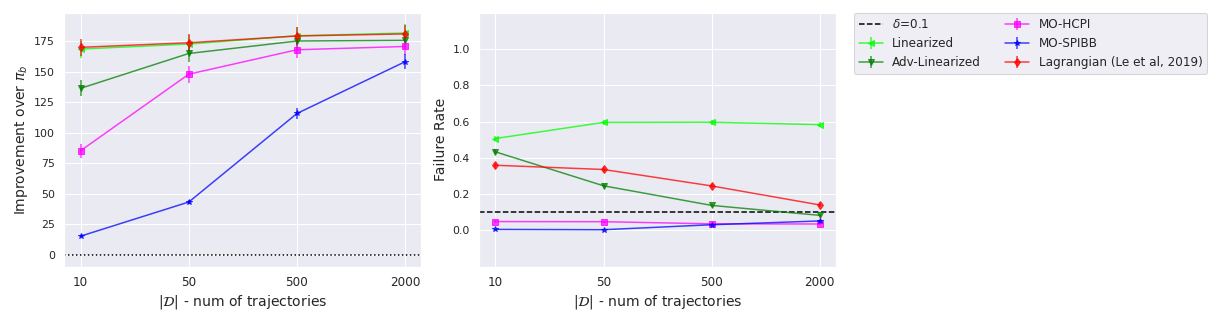}
    \caption{Comparisons of Lagrangian \citep{le2019batch} with $\eta = 0.01$ and MO-SPIBB (\ref{eq:s-opt}) with $\epsilon=0.1$.}
    \label{fig:lag-only-single-sopt} 
\end{subfigure}
\\
\begin{subfigure}[b]{1\textwidth}
    \includegraphics[width=1\textwidth]{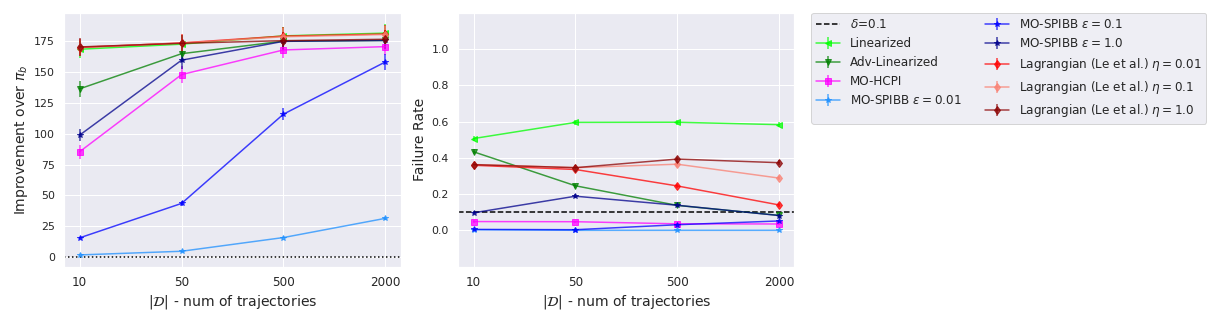}
    \caption{MO-SPIBB (\cref{eq:s-opt}) and Lagrangian \citep{le2019batch} comparisons across different hyper-parameters.}
    \label{fig:lag-multiple-sopt}
\end{subfigure}
\caption[]{
\small
Results on 100 random CMDPs for different $\bml$ and $\rho$ combinations with $\delta=0.1$. The different agents are represented by different markers and colored lines. Each point on the plot denotes the mean (with standard error bars) for 12 different $\bml,\rho$ combinations for the 100 randomly generated CMDPs (1200 datapoints).  The x-axis denotes the amount of data the agents were trained on. 
The y-axis for left subplot in each sub-figure represents the improvement over baseline and the right subplot denotes the failure rate. The dotted black line in the right subplots represents the high-confidence parameter $\delta=0.1$.
\Cref{fig:lag-only-single-sopt} denotes the case when MO-SPIBB (\ref{eq:s-opt}) is run with $\epsilon=0.1$, MO-HCPI (\ref{eq:h-opt}) with $\IS=$ Doubly Robust (DR) estimator with student's t-test concentration inequality, and Lagrangian \citep{le2019batch} with $\eta = 0.01$ . 
\Cref{fig:lag-multiple-sopt} shows how MO-SPIBB and Lagrangian perform across different hyper-parameters.
\label{fig:lag-combined-results}}
\vskip -0.1in
\end{figure}

We test the method by \cite{le2019batch} (henceforth referred to as Lagrangian) in the synthetic navigation CMDP task described in \Cref{sec:synthetic-experiments}. In \Cref{fig:lag-only-single-sopt}, we present the results for the best performing Lagrangian baseline on 100 random CMDPs for different $\bml$ and $\rho$ combinations with $\delta=0.1$. Similar to \Cref{fig:delta-params-mean}, we provide a more detailed plot of how the Lagrangian baseline performs with different hyper-parameters in the above setting in \Cref{fig:lag-multiple-sopt}.


\textbf{Results:} As expected, we observe that the Lagrangian baseline has a high failure rate, particularly in the low-data setting. 
This makes sense as the guarantees provided by \cite{le2019batch} are of the form $\mathcal{J}^\pi_{k,m^{\star}} - \mathcal{J}^{\pi_{b}}_{k, m^{\star}} \geq - \frac{C}{(1-\gamma)^{3/2}}$ (Theorem 4.4 of \cite{le2019batch}), where $C$ is a term that depends on a constant that comes from the Concentrability assumption (Assumption 1 of \cite{le2019batch}). This assumption upper bounds the ratio between the future state-action distributions of any non-stationary policy and the baseline policy under which the dataset was generated by some constant. In other words, it makes assumptions on the quality of the data gathered under the baseline policy. Unfortunately, this assumption cannot be verified in practice, and it is unclear how to get a tractable estimate of this constant. As such, this constant can be arbitrarily large (even infinite) when the baseline policy fails to cover the support of all non-stationary policies, for instance, when the baseline policy is not exploratory enough or when the size of the dataset is small. Hence, we observe a high failure rate of \cite{le2019batch} in the experiments, especially in the low data setting. Compared to \cite{le2019batch}, our performance guarantees do not make any assumptions on the quality of the dataset or the baseline. Therefore, our approach can ensure a low failure rate even in the low-data regime.

\textbf{Implementation details and Hyper-parameters:} We build on top of the publicly available code of \cite{le2019batch} released by the authors and extend it to our setting. 
We are confident that our implementation is correct as we made sure it passes various sanity tests such as convergence of the primal-dual gap and feasibility on access to true MDP parameters.

The algorithm in \cite{le2019batch} (Algorithm 2, Constrained Batch Policy Learning) requires the following hyper-parameters:

\begin{itemize}
    \item Online Learning Subroutine: We use the same online learning algorithm as used by the authors in their experiments, i.e. Exponentiated Gradient \citep{kivinen1997exponentiated}.
    
    \item Duality gap $\omega$ : This denotes the primal-dual gap or the early termination condition. We tried the values in $\set{0.01, 0.001}$ and fix the value to $0.01$.
    
    \item Number of iterations: This parameter denotes the number of iterations for which the Lagrange coefficients should be updated. We experimented in the range $\set{100, 250, 500}$ and set this to $250$.
    
    \item Norm bound $B$: The bound on the norm of Lagrange coefficients vector. We tried the values in $\set{1, 10, 50, 100}$ and fixed it $10$.
    
    \item Learning rate $\eta$: This parameter denotes the learning rate for the update of the Lagrange coefficients via the online learning subroutine. We found that this is the most sensitive variable and we tried with values in $\set{0.005, 0.01, 0.05, 0.1, 0.5, 1.0, 5.0}$. For the final experiments, we benchmark with three different values $(0.01, 0.1, 1.0)$ as mentioned in the \Cref{fig:lag-multiple-sopt}.
    
\end{itemize}

We would like to point out that the hyper-parameter tuning for the Lagrangian baseline can be particularly challenging as in the low-data setting none of the combinations of the above hyper-parameters can ensure a low failure rate even though the duality gap has converged. 


The above experiments show the advantage of our approach over \cite{le2019batch}, particularly in the low-data safety-critical tasks, where our methods can improve over the baseline policy while ensuring a low failure rate.

\subsection{Scaling experiments with number of objectives $d$}
\label{app:cmdp-scaling-experiments}

We experimented with the different number of objectives $d$ to validate if the trends we observed for \ref{eq:s-opt} and \ref{eq:h-opt} in \Cref{sec:synthetic-experiments} also extend to $d>2$. 
In the CMDP formulation, as there can only be one primary reward, we extend the CMDP to include more than 1 type of pits. The extended CMDP now has $d-1$ different kinds of pits and corresponding reward functions, where the agent gets a pit reward of $-1$ if the agent steps into a cell containing that particular kind of pit. We relax the CMDP threshold to $c_i = -10.0$ as the CMDP problem gets harder with more number of pits, and a lower threshold makes the problem of finding $\piopt$ of a random CMDP easier. Therefore, the task objective for the agent in the extended CMDP is to reach the goal in the least amount of steps, such that it can only step into at most 10 pits of every different type. 

We use the same experiment methodology from \Cref{sec:synthetic-experiments}. As the focus is to see how the trends scale with $d$, we fix the $\bml$, with $\lambda_0=1.0$ and the rest of $\lambda_{\ge 1}=0.0$.  We compare \ref{eq:s-opt} and \ref{eq:h-opt} over different $|\D|\in \{ 10, 50, 500, 2000\}$, $\rho \in \{0.1, 0.4, 0.7, 0.9\}$, the fixed set of parameters: $\IS$=DR, $\CI=$student's t-test, $\epsilon\in \{0.001, 0.01, 0.1, 1.0\}$, and $\delta=0.1$.

The results over 10 random CMDPs with fixed parameters can be found in \Cref{fig:scale-exp-10-runs-fixed-delta}. We notice that the trends from \Cref{sec:synthetic-experiments} case still carry till $d\le 1+16$, where for some value of $\epsilon$, \ref{eq:s-opt} can lead to better improvement in $\pib$ while still having failure rate $<\delta$. However, $d > 1+16$ we see there are no obvious trends and both \ref{eq:s-opt} and \ref{eq:h-opt} tend to become very conservative and returning the baseline becomes the best solution choice.

\begin{figure*}
  \includegraphics[width=\textwidth]{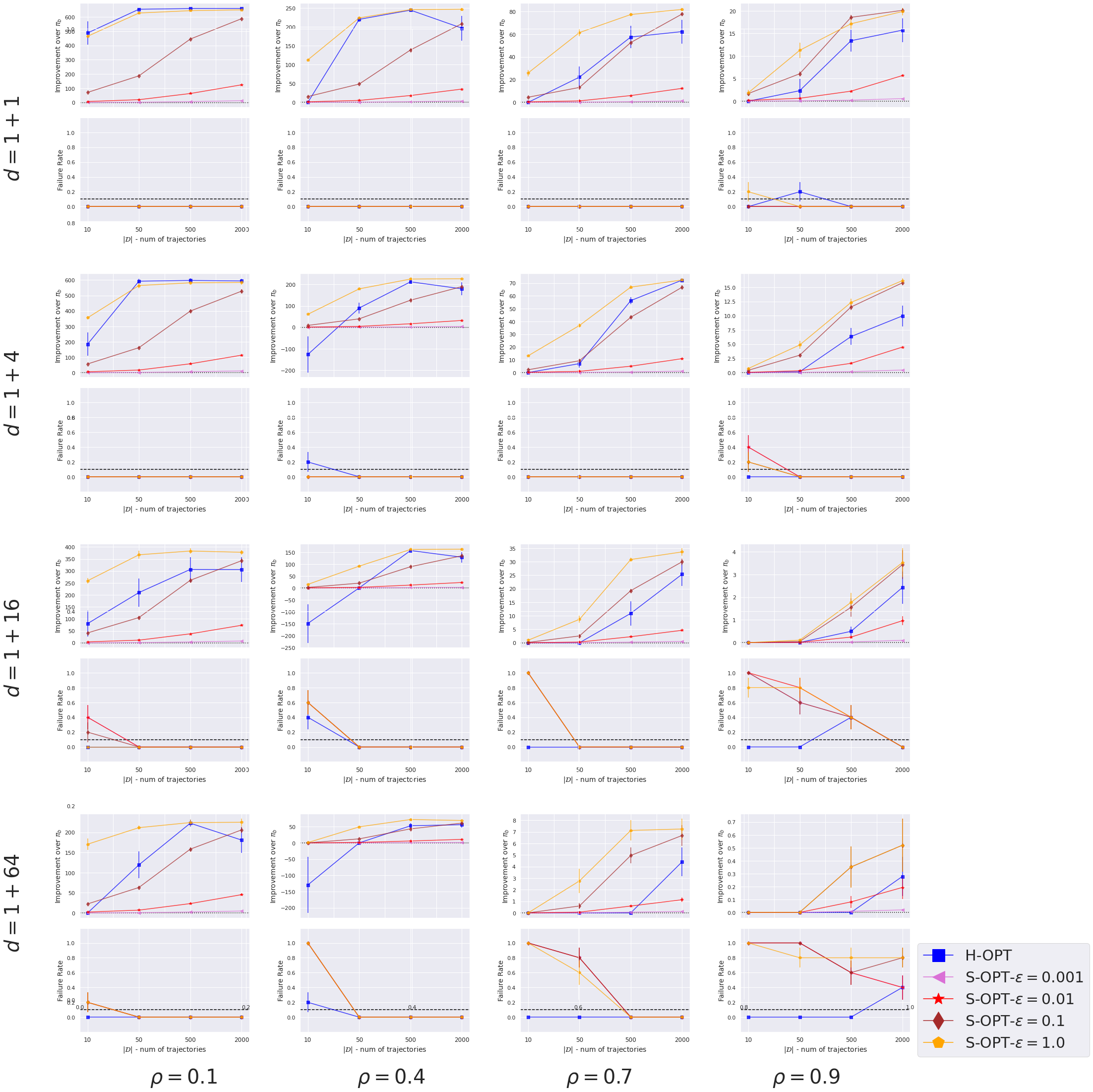}
  \caption{Scaling with $d$ with results on a 10 random CMDPs and $\delta=0.1$. The different agents are represented by different markers and color. Each point on the graph denotes the mean for 100 runs, the standard errors is denoted by the error bars. The x-axis denotes the amount of data the agents were trained on. They y-axis for the top plot in a grid represents the improvement over baseline and the y-axis for bottom plot denotes the failure rate.
  }
  \label{fig:scale-exp-10-runs-fixed-delta}
\end{figure*}

\subsection{Additional details}

For the experiments in \Cref{sec:synthetic-experiments}, on an Intel(R) Xeon(R) Gold 6230 CPU (2.10GHz), the baselines take around 3 seconds to run, and both \ref{eq:s-opt} and \ref{eq:h-opt} take about 5 seconds.

%% file: doc/appendix/sepsis_additional_details.tex
\section{Additional details for sepsis experiments}
\label{app:sepsis-details}

\subsection{Sepsis data and cohort details}
\label{app:sepsis-dataset}

We followed the pre-processing methodology from \cite{tang2020clinician, komorowski2018artificial} and we refer the reader to the original work for more details.

The dosage of prescribed IV fluids and vasopressors is converted into discrete variables to be used as actions for the constructed MDP. Each type of action (IV or vasopressor) is divided into 4 bins (each representing one quantile), and an additional action for "No drug" (0 dose) is also introduced. As such, the $|\A| = 5 \times 5$.
The cohort statistics can be found in \Cref{table:cohort-stats}. The patient data consists of 48 dimensional time-series with features representing attributes such as demographics, vitals and lab work results (\Cref{table:sepsis-features-summary}). 
The patient data is discretized into 4-hour windows, each of which is pre-processed to be treated as a single time-step. The state-space and is discretized using a k-means based clustering algorithm to map the states to $750$ clusters. Two additional absorbing states are added for death and survival ($|\X|=752$).

\begin{table}[h]
    \centering
    \caption{Cohort statistics after following the data pre-processing methodology from \cite{tang2020clinician, komorowski2018artificial}.}
    \label{table:cohort-stats}
    \vskip 0.1in
    \begin{tabular}{lrrrr}
    \toprule
     Survivors &     N & \% Female & Mean Age & Hours in ICU \\
    \midrule
     Survivors & 18066 &    44.5\% &     64.1 &         56.6 \\
    Non-survivors &  2888 &    42.9\% &     68.8 &         60.9 \\
    \bottomrule
\end{tabular}
\end{table}


\begin{table*}[h]
    \centering
    \caption{Summary of the patient state features from \citep[][Table 3]{tang2020clinician}.}
    \label{table:sepsis-features-summary}
    \begin{tabular}{lp{9cm}}
        \toprule
        Demographics/Static &  Age, Gender, SOFA, Shock Index, Elixhauser, SIRS, Re-admission, GCS - Glasgow Coma Scale \\
        \midrule
        Lab values & Albumin, Arterial pH, Calcium, Glucose, Hemoglobin, Magnesium, PTT - Partial Thromboplastin Time, Potassium, SGPT - Serum Glutamic-Pyruvic Transaminase, Arterial Blood Gas, Blood Urea Nitrogen, Chloride, Bicarbonate, International Normalized Ratio, Sodium, Arterial Lactate, CO2, Creatinine, Ionised Calcium, Prothrombin Time, Platelets Count, SGOT - Serum Glutamic-Oxaloacetic Transaminase, Total bilirubin, White Blood Cell Count \\
        \midrule
        Vital signs & Diastolic Blood Pressure, Systolic Blood Pressure, Mean Blood Pressure, PaCO2, PaO2, FiO2, PaO/FiO2 ratio, Respiratory Rate, Temperature (Celsius), Weight (kg), Heart Rate, SpO2 \\
        \midrule
        Intake and output events & Fluid Output - 4 hourly period, Total Fluid Output, Mechanical Ventilation \\
        \bottomrule
    \end{tabular}

\end{table*}

\subsection{Performance on changing hyper-parameters}
\label{app:sepsis-hyperparams}

For the experiments in \Cref{sec:sepsis-experiments}, we treat $\delta$ as a hyper-parameter. For \ref{eq:s-opt} instead of searching over both $\delta$ and $\epsilon$, we follow the strategy proposed in Soft-SPIBB: fix the $\delta=1.0$ and only search over $\epsilon$. 
For \ref{eq:h-opt}, we found that only DR and WDR gave reliable off-policy estimates so report the results with both of them with different $\delta$. As in previous sections, we used student's t-test as the choice of concentration inequality for \ref{eq:h-opt}.

\subsubsection{\ref{eq:s-opt} parameters}

Here, we fix $\delta=1.0$ and try with different values of the hyper-parameter $\epsilon$ and directly report the results directly on the test set. The results are presented in \Cref{table:app-s-opt-test}.

\subsubsection{\ref{eq:h-opt} parameters}

We run with different values of the hyper-parameter $\delta$ and directly report the results directly on the test set for different $\IS$ estimators. The results for DR estimator are presented in \Cref{table:app-hopt-DR-Adv} and for WDR estimator are presented in \Cref{table:app-hopt-WDR-Adv}.


\subsection{Additional qualitative Analysis}
\label{app:sepsis-qual-analysis}
We calculate how many rare-actions are recommended by different solution policies and compare them with the most common actions taken by the clinicians.
For each state, for the action recommended by a solution policy, we calculate the frequency with which that state-action was observed in the training data and calculate the percentage of time that state-action pair was observed among all the possible actions taken from that state.
Across all the states, the actions suggested by the traditional single-objective RL baseline are observed only 3\% of the time on average (5.3 observations per state). Whereas, the actions most commonly chosen by the clinicians  are observed 51.4\% of the time on average (138.2 observations per state). We study this behavior for two of the policies returned by MO-SPIBB that deviate the most from the baseline: for the policy returned by \ref{eq:s-opt} ($\bml=[1,1]$) the recommended actions are observed 24.8\% of time on average (61.0 observations per state) and for  \ref{eq:s-opt} ($\bml=[0,1]$) the recommended actions are observed 23.4\% of times (56.14 observations per state).

\subsection{Additional details}

For the experiments in \Cref{sec:sepsis-experiments}, on an Intel(R) Xeon(R) Gold 6230 CPU (2.10GHz), running the Linearized baseline takes around 30 seconds, Adv-Linearized takes around 60 seconds, \ref{eq:s-opt} take about 90-120 seconds and \ref{eq:h-opt} takes about 90 seconds.


\begin{table}[h]
    \centering
    \caption{
    Performance of various \ref{eq:s-opt} policy candidates (with different $\epsilon$) using DR and WDR estimation with standard errors on 10 random splits of the TEST dataset. 
    The red cells denote the corresponding safety constraint violation, i.e, either $\mathcal{J}_{0}^{\pi} < \mathcal{J}_{0}^{\pib}$ or $-\mathcal{J}_{1}^{\pi} > -\mathcal{J}_{1}^{\pib}$.}
    \label{table:app-s-opt-test}
    \vskip 0.1in
    \begin{adjustbox}{max width=1\textwidth,center}
    \begin{tabular}{cccccc}
    \toprule
    \multicolumn{1}{c}{User preferences $(\bml)$} & \multicolumn{1}{c}{Policy} & \multicolumn{2}{c}{Survival return ($\mathcal{J}_0$)} & \multicolumn{2}{c}{Rare-treatment return ($- \mathcal{J}_1$)} \\
    \hline
    & & DR & WDR & DR & WDR  \\  \cline{3-6}
    & Clinician's ($\pib$) & 64.78 $\pm$ 0.90 & 64.78 $\pm$ 0.90          & 13.58 $\pm$ 0.19 & 13.58 $\pm$ 0.19  \\
    \midrule
    \multirow{4}{*}{$[\lambda_0=1.0, \lambda_1 = 0.0]$} 
& Linearized & 97.68 $\pm$ 0.22 & 97.58 $\pm$ 0.20   & \textcolor{red}{27.64 $\pm$ 1.11 }& \textcolor{red}{27.84 $\pm$ 1.09 } \\ 
& \ref{eq:s-opt}, $\epsilon=0.0$  & 64.78 $\pm$ 0.90 & 64.78 $\pm$ 0.90   & 13.58 $\pm$ 0.19 & 13.58 $\pm$ 0.19  \\
& \ref{eq:s-opt}, $\epsilon=0.001$  & 64.91 $\pm$ 0.90 & 64.91 $\pm$ 0.90   & 13.56 $\pm$ 0.19 & 13.56 $\pm$ 0.19  \\
& \ref{eq:s-opt}, $\epsilon=0.01$  & 66.11 $\pm$ 0.87 & 66.05 $\pm$ 0.86   & 13.42 $\pm$ 0.20 & 13.46 $\pm$ 0.20  \\
& \ref{eq:s-opt}, $\epsilon=0.1$  & 73.70 $\pm$ 0.84 & 71.96 $\pm$ 0.69   & 12.30 $\pm$ 0.39 & \textcolor{red}{13.80 $\pm$ 0.33 }  \\
& \ref{eq:s-opt}, $\epsilon=0.5$  & 78.19 $\pm$ 0.54 & 81.01 $\pm$ 0.36   & \textcolor{red}{16.21 $\pm$ 0.49 }& 13.10 $\pm$ 0.31  \\
& \ref{eq:s-opt}, $\epsilon=1.0$  & 84.03 $\pm$ 0.48 & 87.11 $\pm$ 0.33   & \textcolor{red}{15.54 $\pm$ 0.59 }& 12.17 $\pm$ 0.59  \\
& \ref{eq:s-opt}, $\epsilon=2.5$  & 90.05 $\pm$ 0.25 & 91.37 $\pm$ 0.20   & \textcolor{red}{15.35 $\pm$ 0.72 }& 13.53 $\pm$ 0.56  \\
& \ref{eq:s-opt}, $\epsilon=5.0$  & 91.58 $\pm$ 0.49 & 92.66 $\pm$ 0.28   & \textcolor{red}{15.39 $\pm$ 0.59 }& \textcolor{red}{13.71 $\pm$ 0.38 }\\
& \ref{eq:s-opt}, $\epsilon=10.0$  & 91.64 $\pm$ 0.47 & 92.68 $\pm$ 0.23   & \textcolor{red}{15.19 $\pm$ 0.59 }& 13.56 $\pm$ 0.42  \\
& \ref{eq:s-opt}, $\epsilon=\infty$  & 91.62 $\pm$ 0.46 & 92.68 $\pm$ 0.23   & \textcolor{red}{15.18 $\pm$ 0.59 }& 13.56 $\pm$ 0.42  \\
    \midrule
    \multirow{4}{*}{$[\lambda_0=1.0, \lambda_1 = 1.0]$}
    & Linearized & 87.17 $\pm$ 0.48 & 89.11 $\pm$ 0.37   & 2.41 $\pm$ 0.47 & 1.52 $\pm$ 0.41\\
& \ref{eq:s-opt}, $\epsilon=0.0$  & 64.78 $\pm$ 0.90 & 64.78 $\pm$ 0.90   & 13.58 $\pm$ 0.19 & 13.58 $\pm$ 0.19\\
& \ref{eq:s-opt}, $\epsilon=0.001$  & 64.90 $\pm$ 0.90 & 64.90 $\pm$ 0.90   & 13.53 $\pm$ 0.19 & 13.54 $\pm$ 0.19\\
& \ref{eq:s-opt}, $\epsilon=0.01$  & 66.02 $\pm$ 0.88 & 65.94 $\pm$ 0.87   & 13.15 $\pm$ 0.20 & 13.20 $\pm$ 0.20\\
& \ref{eq:s-opt}, $\epsilon=0.1$  & 74.34 $\pm$ 0.78 & 72.04 $\pm$ 0.87   & 9.32 $\pm$ 0.29 & 10.48 $\pm$ 0.45\\
& \ref{eq:s-opt}, $\epsilon=0.5$  & 76.47 $\pm$ 0.50 & 78.42 $\pm$ 0.41   & 7.61 $\pm$ 0.44 & 5.02 $\pm$ 0.17\\
& \ref{eq:s-opt}, $\epsilon=1.0$  & 81.39 $\pm$ 0.46 & 84.54 $\pm$ 0.36   & 4.64 $\pm$ 0.40 & 2.38 $\pm$ 0.22 \\
& \ref{eq:s-opt}, $\epsilon=2.5$  & 86.26 $\pm$ 0.33 & 88.09 $\pm$ 0.24   & 1.98 $\pm$ 0.28 & 1.14 $\pm$ 0.27  \\
& \ref{eq:s-opt}, $\epsilon=5.0$  & 86.76 $\pm$ 0.47 & 88.55 $\pm$ 0.22   & 2.52 $\pm$ 0.48 & 1.55 $\pm$ 0.41\\
& \ref{eq:s-opt}, $\epsilon=10.0$  & 86.77 $\pm$ 0.49 & 88.58 $\pm$ 0.25   & 2.53 $\pm$ 0.50 & 1.57 $\pm$ 0.43  \\
& \ref{eq:s-opt}, $\epsilon=\infty$  & 86.77 $\pm$ 0.49 & 88.58 $\pm$ 0.25   & 2.53 $\pm$ 0.50 & 1.57 $\pm$ 0.43  \\
    \midrule
    \multirow{4}{*}{$[\lambda_0=0.0, \lambda_1 = 0.0]$}
    & Linearized & \textcolor{red}{-89.39 $\pm$ 0.43} & \textcolor{red}{-90.90 $\pm$ 0.29 }  & \textcolor{red}{22.99 $\pm$ 0.40 }& \textcolor{red}{22.81 $\pm$ 0.30 }  \\ 
    & \ref{eq:s-opt}, $\epsilon=0.0$  & 64.78 $\pm$ 0.90 & 64.78 $\pm$ 0.90   & 13.58 $\pm$ 0.19 & 13.58 $\pm$ 0.19\\
& \ref{eq:s-opt}, $\epsilon=0.001$  & 64.80 $\pm$ 0.90 & 64.80 $\pm$ 0.90   & 13.57 $\pm$ 0.19 & 13.57 $\pm$ 0.19\\
& \ref{eq:s-opt}, $\epsilon=0.01$  & 64.92 $\pm$ 0.90 & 64.92 $\pm$ 0.90   & 13.50 $\pm$ 0.19 & 13.51 $\pm$ 0.19\\
& \ref{eq:s-opt}, $\epsilon=0.1$  & 65.78 $\pm$ 0.89 & 65.70 $\pm$ 0.88   & 13.20 $\pm$ 0.20 & 13.25 $\pm$ 0.20  \\
& \ref{eq:s-opt}, $\epsilon=0.5$  & 67.73 $\pm$ 0.82 & 67.22 $\pm$ 0.88   & 13.24 $\pm$ 0.24 & 13.55 $\pm$ 0.33\\
& \ref{eq:s-opt}, $\epsilon=1.0$  & 69.12 $\pm$ 0.75 & 67.90 $\pm$ 0.84   & 13.57 $\pm$ 0.27 & \textcolor{red}{14.39 $\pm$ 0.44 }\\
& \ref{eq:s-opt}, $\epsilon=2.5$  & 71.00 $\pm$ 0.63 & 68.28 $\pm$ 0.46   & \textcolor{red}{14.27 $\pm$ 0.30 }& \textcolor{red}{15.73 $\pm$ 0.40 }  \\
& \ref{eq:s-opt}, $\epsilon=5.0$  & 71.95 $\pm$ 0.54 & 69.27 $\pm$ 0.63   & \textcolor{red}{15.29 $\pm$ 0.39 }& \textcolor{red}{16.12 $\pm$ 0.70 }  \\
& \ref{eq:s-opt}, $\epsilon=10.0$  & 72.73 $\pm$ 0.64 & 71.17 $\pm$ 0.65   & \textcolor{red}{16.59 $\pm$ 0.37 }& \textcolor{red}{16.21 $\pm$ 0.41 }\\
& \ref{eq:s-opt}, $\epsilon=\infty$  & \textcolor{red}{60.27 $\pm$ 0.49} & \textcolor{red}{61.44 $\pm$ 0.85 }  & \textcolor{red}{18.40 $\pm$ 0.27 }& \textcolor{red}{15.36 $\pm$ 0.58 }  \\
    \midrule
    \multirow{4}{*}{$[\lambda_0=0.0, \lambda_1 = 1.0]$}
& Linearized & \textcolor{red}{58.27 $\pm$ 2.18} & \textcolor{red}{60.52 $\pm$ 2.07 }  & 0.04 $\pm$ 0.03 & 0.02 $\pm$ 0.01  \\ 
    & \ref{eq:s-opt}, $\epsilon=0.0$  & 64.78 $\pm$ 0.90 & 64.78 $\pm$ 0.90   & 13.58 $\pm$ 0.19 & 13.58 $\pm$ 0.19  \\
& \ref{eq:s-opt}, $\epsilon=0.001$  & 64.83 $\pm$ 0.90 & 64.83 $\pm$ 0.90   & 13.52 $\pm$ 0.19 & 13.52 $\pm$ 0.19  \\
 & \ref{eq:s-opt}, $\epsilon=0.01$  & 65.36 $\pm$ 0.88 & 65.27 $\pm$ 0.88   & 12.96 $\pm$ 0.19 & 13.01 $\pm$ 0.19  \\
 & \ref{eq:s-opt}, $\epsilon=0.1$  & 71.35 $\pm$ 0.96 & 69.29 $\pm$ 0.92   & 7.75 $\pm$ 0.19 & 8.30 $\pm$ 0.18\\
 & \ref{eq:s-opt}, $\epsilon=0.5$  & 71.01 $\pm$ 0.72 & 71.30 $\pm$ 0.68   & 2.54 $\pm$ 0.37 & 1.50 $\pm$ 0.11  \\
 & \ref{eq:s-opt}, $\epsilon=1.0$  & 74.19 $\pm$ 0.57 & 76.11 $\pm$ 0.57   & 0.90 $\pm$ 0.14 & 0.34 $\pm$ 0.09  \\
 & \ref{eq:s-opt}, $\epsilon=2.5$  & 76.42 $\pm$ 0.61 & 77.20 $\pm$ 0.72   & 0.10 $\pm$ 0.06 & 0.06 $\pm$ 0.04  \\
 & \ref{eq:s-opt}, $\epsilon=5.0$  & 76.08 $\pm$ 0.65 & 76.87 $\pm$ 0.74   & 0.07 $\pm$ 0.05 & 0.05 $\pm$ 0.03\\
 & \ref{eq:s-opt}, $\epsilon=10.0$  & 76.07 $\pm$ 0.65 & 76.87 $\pm$ 0.73   & 0.07 $\pm$ 0.05 & 0.04 $\pm$ 0.03\\
 & \ref{eq:s-opt}, $\epsilon=\infty$  & 76.05 $\pm$ 0.65 & 76.85 $\pm$ 0.72   & 0.07 $\pm$ 0.05 & 0.04 $\pm$ 0.03\\
 \bottomrule
    \addtocounter{table}{-1} 
    \end{tabular}
    \end{adjustbox}
\end{table}

\begin{table}[h]
    \centering
    \caption{
    Performance of various \ref{eq:h-opt} policy candidates (with different $\delta$) using $\IS=$ DR estimator with standard errors on 10 random splits of the TEST dataset. 
    The red cells denote the corresponding safety constraint violation, i.e, either $\mathcal{J}_{0}^{\pi} < \mathcal{J}_{0}^{\pib}$ or $-\mathcal{J}_{1}^{\pi} > -\mathcal{J}_{1}^{\pib}$.}
    \label{table:app-hopt-DR-Adv}
    \vskip 0.1in
    \begin{adjustbox}{max width=1\textwidth,center}
    \begin{tabular}{cccccc}
    \toprule
    \multicolumn{1}{c}{User preferences $(\bml)$} & \multicolumn{1}{c}{Policy} & \multicolumn{2}{c}{Survival return ($\mathcal{J}_0$)} & \multicolumn{2}{c}{Rare-treatment return ($- \mathcal{J}_1$)} \\
    \hline
    & & DR & WDR & DR & WDR  \\  \cline{3-6}
    & Clinician's ($\pib$) & 64.78 $\pm$ 0.90 & 64.78 $\pm$ 0.90          & 13.58 $\pm$ 0.19 & 13.58 $\pm$ 0.19  \\
    \midrule
    \multirow{4}{*}{$[\lambda_0=1.0, \lambda_1 = 0.0]$} 
    & Linearized & 97.68 $\pm$ 0.22 & 97.58 $\pm$ 0.20   & \textcolor{red}{27.64 $\pm$ 1.11 }& \textcolor{red}{27.84 $\pm$ 1.09 } \\ 
    & \ref{eq:h-opt}, $\delta=0.1$  & 65.95 $\pm$ 0.00 & 65.95 $\pm$ 0.00   & 13.37 $\pm$ 0.00 & 13.37 $\pm$ 0.00\\
    & \ref{eq:h-opt},  $\delta=0.3$  & 65.95 $\pm$ 0.00 & 65.95 $\pm$ 0.00   & 13.37 $\pm$ 0.00 & 13.37 $\pm$ 0.00\\
    & \ref{eq:h-opt}, $\delta=0.5$  & 65.95 $\pm$ 0.00 & 65.95 $\pm$ 0.00   & 13.37 $\pm$ 0.00 & 13.37 $\pm$ 0.00\\
    & \ref{eq:h-opt}, $\delta=0.7$  & 65.95 $\pm$ 0.00 & 65.95 $\pm$ 0.00   & 13.37 $\pm$ 0.00 & 13.37 $\pm$ 0.00\\
    & \ref{eq:h-opt}, $\delta=0.9$  & 65.95 $\pm$ 0.00 & 65.95 $\pm$ 0.00   & 13.37 $\pm$ 0.00 & 13.37 $\pm$ 0.00\\
    \midrule
    \multirow{4}{*}{$[\lambda_0=1.0, \lambda_1 = 1.0]$}
    & Linearized & 87.17 $\pm$ 0.48 & 89.11 $\pm$ 0.37   & 2.41 $\pm$ 0.47 & 1.52 $\pm$ 0.41\\
    & \ref{eq:h-opt}, $\delta=0.1$  & 86.37 $\pm$ 0.00 & 88.03 $\pm$ 0.00   & 2.58 $\pm$ 0.00 & 1.43 $\pm$ 0.00\\
    & \ref{eq:h-opt}, $\delta=0.3$  & 86.37 $\pm$ 0.00 & 88.03 $\pm$ 0.00   & 2.58 $\pm$ 0.00 & 1.43 $\pm$ 0.00\\
    & \ref{eq:h-opt}, $\delta=0.5$  & 86.37 $\pm$ 0.00 & 88.03 $\pm$ 0.00   & 2.58 $\pm$ 0.00 & 1.43 $\pm$ 0.00\\
    & \ref{eq:h-opt}, $\delta=0.7$  & 86.37 $\pm$ 0.00 & 88.03 $\pm$ 0.00   & 2.58 $\pm$ 0.00 & 1.43 $\pm$ 0.00\\
    & \ref{eq:h-opt}, $\delta=0.9$  & 86.37 $\pm$ 0.00 & 88.03 $\pm$ 0.00   & 2.58 $\pm$ 0.00 & 1.43 $\pm$ 0.00  \\
    \midrule
    \multirow{4}{*}{$[\lambda_0=0.0, \lambda_1 = 0.0]$}
    & Linearized & \textcolor{red}{-89.39 $\pm$ 0.43} & \textcolor{red}{-90.90 $\pm$ 0.29 }  & \textcolor{red}{22.99 $\pm$ 0.40 }& \textcolor{red}{22.81 $\pm$ 0.30 }  \\ 
    & \ref{eq:h-opt}, $\delta=0.1$  & 65.95 $\pm$ 0.00 & 65.95 $\pm$ 0.00   & 13.37 $\pm$ 0.00 & 13.37 $\pm$ 0.00  \\
    & \ref{eq:h-opt}, $\delta=0.3$  & 65.95 $\pm$ 0.00 & 65.95 $\pm$ 0.00   & 13.37 $\pm$ 0.00 & 13.37 $\pm$ 0.00\\
    & \ref{eq:h-opt}, $\delta=0.5$  & 65.95 $\pm$ 0.00 & 65.95 $\pm$ 0.00   & 13.37 $\pm$ 0.00 & 13.37 $\pm$ 0.00\\
    & \ref{eq:h-opt}, $\delta=0.7$  & 68.28 $\pm$ 0.00 & \textcolor{red}{63.25 $\pm$ 0.00 }  & \textcolor{red}{14.16 $\pm$ 0.00 }& \textcolor{red}{16.41 $\pm$ 0.00 } \\
    & \ref{eq:h-opt}, $\delta=0.9$  & 68.28 $\pm$ 0.00 & \textcolor{red}{63.25 $\pm$ 0.00 }  & \textcolor{red}{14.16 $\pm$ 0.00 }& \textcolor{red}{16.41 $\pm$ 0.00 }\\
    \midrule
    \multirow{4}{*}{$[\lambda_0=0.0, \lambda_1 = 1.0]$}
    & Linearized & \textcolor{red}{58.27 $\pm$ 2.18} & \textcolor{red}{60.52 $\pm$ 2.07 }  & 0.04 $\pm$ 0.03 & 0.02 $\pm$ 0.01  \\ 
    & \ref{eq:h-opt}, $\delta=0.1$  & 76.54 $\pm$ 0.00 & 77.55 $\pm$ 0.00   & 0.09 $\pm$ 0.00 & 0.05 $\pm$ 0.00\\
    & \ref{eq:h-opt}, $\delta=0.3$  & 76.54 $\pm$ 0.00 & 77.55 $\pm$ 0.00   & 0.09 $\pm$ 0.00 & 0.05 $\pm$ 0.00\\
    & \ref{eq:h-opt}, $\delta=0.5$  & 76.54 $\pm$ 0.00 & 77.55 $\pm$ 0.00   & 0.09 $\pm$ 0.00 & 0.05 $\pm$ 0.00\\
    & \ref{eq:h-opt}, $\delta=0.7$  & 76.54 $\pm$ 0.00 & 77.55 $\pm$ 0.00   & 0.09 $\pm$ 0.00 & 0.05 $\pm$ 0.00\\
    & \ref{eq:h-opt}, $\delta=0.9$  & 76.54 $\pm$ 0.00 & 77.55 $\pm$ 0.00   & 0.09 $\pm$ 0.00 & 0.05 $\pm$ 0.00  \\
    \bottomrule
    \addtocounter{table}{-1} 
    \end{tabular}
    \end{adjustbox}
\end{table}

\begin{table}[h]
    \centering
    \caption{
    Performance of various \ref{eq:h-opt} policy candidates (with different $\delta$) using $\IS=$Weighed DR (WDR) estimator with standard errors on 10 random splits of the TEST dataset. 
    The red cells denote the corresponding safety constraint violation, i.e, either $\mathcal{J}_{0}^{\pi} < \mathcal{J}_{0}^{\pib}$ or $-\mathcal{J}_{1}^{\pi} > -\mathcal{J}_{1}^{\pib}$.}
    \label{table:app-hopt-WDR-Adv}
    \vskip 0.1in
    \begin{adjustbox}{max width=1\textwidth,center}
    \begin{tabular}{cccccc}
    \toprule
    \multicolumn{1}{c}{User preferences $(\bml)$} & \multicolumn{1}{c}{Policy} & \multicolumn{2}{c}{Survival return ($\mathcal{J}_0$)} & \multicolumn{2}{c}{Rare-treatment return ($- \mathcal{J}_1$)} \\
    \hline
    & & DR & WDR & DR & WDR  \\  \cline{3-6}
    & Clinician's ($\pib$) & 64.78 $\pm$ 0.90 & 64.78 $\pm$ 0.90          & 13.58 $\pm$ 0.19 & 13.58 $\pm$ 0.19  \\
    \midrule
    \multirow{4}{*}{$[\lambda_0=1.0, \lambda_1 = 0.0]$} 
    & Linearized & 97.68 $\pm$ 0.22 & 97.58 $\pm$ 0.20   & \textcolor{red}{27.64 $\pm$ 1.11 }& \textcolor{red}{27.84 $\pm$ 1.09 } \\ 
    & \ref{eq:h-opt}, $\delta=0.1$  & 65.95 $\pm$ 0.00 & 65.95 $\pm$ 0.00   & 13.37 $\pm$ 0.00 & 13.37 $\pm$ 0.00\\
    & \ref{eq:h-opt}, $\delta=0.3$  & 65.95 $\pm$ 0.00 & 65.95 $\pm$ 0.00   & 13.37 $\pm$ 0.00 & 13.37 $\pm$ 0.00\\
    & \ref{eq:h-opt}, $\delta=0.5$  & 65.95 $\pm$ 0.00 & 65.95 $\pm$ 0.00   & 13.37 $\pm$ 0.00 & 13.37 $\pm$ 0.00\\
    & \ref{eq:h-opt}, $\delta=0.7$  & 65.95 $\pm$ 0.00 & 65.95 $\pm$ 0.00   & 13.37 $\pm$ 0.00 & 13.37 $\pm$ 0.00\\
    & \ref{eq:h-opt}, $\delta=0.9$  & 91.39 $\pm$ 0.00 & 92.61 $\pm$ 0.00   & \textcolor{red}{15.41 $\pm$ 0.00 }& \textcolor{red}{13.89 $\pm$ 0.00 } \\
    \midrule
    \multirow{4}{*}{$[\lambda_0=1.0, \lambda_1 = 1.0]$}
    & Linearized & 87.17 $\pm$ 0.48 & 89.11 $\pm$ 0.37   & 2.41 $\pm$ 0.47 & 1.52 $\pm$ 0.41\\
    & \ref{eq:h-opt}, $\delta=0.1$  & 86.37 $\pm$ 0.00 & 88.03 $\pm$ 0.00   & 2.58 $\pm$ 0.00 & 1.43 $\pm$ 0.00\\
    & \ref{eq:h-opt}, $\delta=0.3$  & 86.37 $\pm$ 0.00 & 88.03 $\pm$ 0.00   & 2.58 $\pm$ 0.00 & 1.43 $\pm$ 0.00\\
    & \ref{eq:h-opt}, $\delta=0.5$  & 86.37 $\pm$ 0.00 & 88.03 $\pm$ 0.00   & 2.58 $\pm$ 0.00 & 1.43 $\pm$ 0.00  \\
    & \ref{eq:h-opt}, $\delta=0.7$  & 86.37 $\pm$ 0.00 & 88.03 $\pm$ 0.00   & 2.58 $\pm$ 0.00 & 1.43 $\pm$ 0.00\\
    & \ref{eq:h-opt}, $\delta=0.9$  & 86.37 $\pm$ 0.00 & 88.03 $\pm$ 0.00   & 2.58 $\pm$ 0.00 & 1.43 $\pm$ 0.00\\
    \midrule
    \multirow{4}{*}{$[\lambda_0=0.0, \lambda_1 = 0.0]$}
    & Linearized & \textcolor{red}{-89.39 $\pm$ 0.43} & \textcolor{red}{-90.90 $\pm$ 0.29 }  & \textcolor{red}{22.99 $\pm$ 0.40 }& \textcolor{red}{22.81 $\pm$ 0.30 }  \\ 
    & \ref{eq:h-opt}, $\delta=0.1$  & 65.95 $\pm$ 0.00 & 65.95 $\pm$ 0.00   & 13.37 $\pm$ 0.00 & 13.37 $\pm$ 0.00  \\
    & \ref{eq:h-opt}, $\delta=0.3$  & 65.95 $\pm$ 0.00 & 65.95 $\pm$ 0.00   & 13.37 $\pm$ 0.00 & 13.37 $\pm$ 0.00\\
    & \ref{eq:h-opt}, $\delta=0.5$  & 65.95 $\pm$ 0.00 & 65.95 $\pm$ 0.00   & 13.37 $\pm$ 0.00 & 13.37 $\pm$ 0.00  \\
    & \ref{eq:h-opt}, $\delta=0.7$  & 65.95 $\pm$ 0.00 & 65.95 $\pm$ 0.00   & 13.37 $\pm$ 0.00 & 13.37 $\pm$ 0.00\\
    & \ref{eq:h-opt}, $\delta=0.9$  & 65.95 $\pm$ 0.00 & 65.95 $\pm$ 0.00   & 13.37 $\pm$ 0.00 & 13.37 $\pm$ 0.00  \\
    \midrule
    \multirow{4}{*}{$[\lambda_0=0.0, \lambda_1 = 1.0]$}
    & Linearized & \textcolor{red}{58.27 $\pm$ 2.18} & \textcolor{red}{60.52 $\pm$ 2.07 }  & 0.04 $\pm$ 0.03 & 0.02 $\pm$ 0.01  \\ 
    & \ref{eq:h-opt}, $\delta=0.1$  & 76.54 $\pm$ 0.00 & 77.55 $\pm$ 0.00   & 0.09 $\pm$ 0.00 & 0.05 $\pm$ 0.00\\
    & \ref{eq:h-opt}, $\delta=0.3$  & 76.54 $\pm$ 0.00 & 77.55 $\pm$ 0.00   & 0.09 $\pm$ 0.00 & 0.05 $\pm$ 0.00\\
    & \ref{eq:h-opt}, $\delta=0.5$  & 76.54 $\pm$ 0.00 & 77.55 $\pm$ 0.00   & 0.09 $\pm$ 0.00 & 0.05 $\pm$ 0.00\\
    & \ref{eq:h-opt}, $\delta=0.7$  & 76.54 $\pm$ 0.00 & 77.55 $\pm$ 0.00   & 0.09 $\pm$ 0.00 & 0.05 $\pm$ 0.00\\
    & \ref{eq:h-opt}, $\delta=0.9$  & 76.54 $\pm$ 0.00 & 77.55 $\pm$ 0.00   & 0.09 $\pm$ 0.00 & 0.05 $\pm$ 0.00\\
    \bottomrule
    \addtocounter{table}{-1} 
    \end{tabular}
    \end{adjustbox}
\end{table}